\documentclass[runningheads, 11pt]{llncs}
\usepackage[a4paper,margin=1.1in]{geometry}
\usepackage[T1]{fontenc}
\usepackage{microtype}
\usepackage{graphicx}
\usepackage{booktabs} 
\usepackage{hyperref}
\usepackage{color}

\usepackage{cite}
\newcommand{\citet}[1]{\cite{#1}}
\newcommand{\citep}[1]{\cite{#1}}
\usepackage[utf8]{inputenc} 
\usepackage[parfill]{parskip}
\usepackage{amssymb}

\usepackage{amsmath, bm, bbm}
\usepackage{mathtools}
\usepackage{cases}

\usepackage{microtype}

\usepackage{algorithm,algorithmic}
\usepackage{color}
\usepackage{appendix}

\usepackage{url}

\usepackage[capitalize,noabbrev]{cleveref}
\crefname{equation}{eq.}{eq.}
\Crefname{equation}{Eq.}{Eq.}

\crefname{definition}{\textbf{definition}}{definitions}
\Crefname{definition}{Definition}{Definitions}
\crefname{assumption}{\textbf{assumption}}{assumptions}
\Crefname{assumption}{Assumption}{Assumptions}

\definecolor{maroon}{RGB}{192,80,77}

\newtheorem{assumption}{Assumption}

\usepackage[textsize=tiny]{todonotes}

\newcommand{\argmax}{\mathop{\mathrm{argmax}}}

\def\E{\mathbb{E}}

\def\R{\mathbb{R}}

\def\cN{\mathcal{N}}

\usepackage{etoolbox}
\usepackage{comment}
\usepackage{caption}
\usepackage{subcaption}
\usepackage{isomath}
\usepackage{indentfirst}
\usepackage{enumerate}
\usepackage{enumitem}
\setlist{leftmargin=10mm}
\usepackage{mathrsfs}
\usepackage{multirow}
\usepackage{amstext}
\usepackage{dsfont}
\def\ind{\mathds{1}}

\newbool{twocol}
\setbool{twocol}{true}
\usepackage{graphicx}
\usepackage{diagbox}

\newbool{compact}  
\setbool{compact}{true}
\usepackage{booktabs}       
\usepackage{amsfonts}       
\usepackage{nicefrac}       
\usepackage{xcolor}         

\begin{document}
\title{Dynamic Pricing with Adversarially-Censored Demands}
%
%
\author{
Jianyu Xu\inst{1} \and Yining Wang\inst{2} \and Xi Chen\inst{3} \and Yu-Xiang Wang\inst{4}
}
%
\institute{Carnegie Mellon University \and University of Texas at Dallas \and New York University \and University of California San Diego\\
}
\maketitle              

\begin{abstract}
We study an online dynamic pricing problem where the potential demand at each time period $t=1,2,\ldots, T$ is stochastic and dependent on the price. However, a perishable inventory is imposed at the beginning of each time $t$, \emph{censoring} the potential demand if it exceeds the inventory level. To address this problem, we introduce a pricing algorithm based on the optimistic estimates of derivatives. We show that our algorithm achieves $\tilde{O}(\sqrt{T})$ optimal regret even with \emph{adversarial} inventory series. Our findings advance the state-of-the-art in online decision-making problems with censored feedback, offering a theoretically optimal solution against adversarial observations.
\end{abstract}
\section{Introduction}
\label{sec:introduction}
The problem of dynamic pricing, where the seller proposes and adjusts their prices over time, has been studied since the seminal work of \citet{cournot1897researches}. The crux to pricing is to balance the profit of sales per unit with the quantity of sales. Therefore, it is imperative for the seller to learn customers' demand as a function of price (commonly known as the \emph{demand curve}) on the fly. However, the demand can often be obfuscated by the observed quantity of sales, especially when \emph{censored} by \emph{inventory} stockouts. Such instances severely impede the seller from learning the underlying demand distributions, thereby hindering our pursuit of the optimal price.

Existing literature has devoted considerable effort to the intersection of pricing and inventory decisions. Such works often consider scenarios with indirectly observable lost demands \citep{keskin2022data}, recoverable leftover demands \citet{chen2019coordinating}, or controllable inventory level \citep{chen2023optimal}. However, these assumptions do not always align with the realities faced in various common business environments. To illustrate, we present two pertinent examples:

\begin{example}[Performance Tickets]\label{example_performance}
    Imagine that we manage a touring company that arranges a series of performances featuring a renowned artist across various cities. Each venue has a different seating capacity, which substantially affects how we set ticket prices. If the price is too high, it may deter attendance, leading to lower revenue. On the other hand, setting it too low could mean that tickets sell out quickly, leaving many potential attendees unable to purchase them. We do not know exactly how many people attempt to buy tickets and fail. Moreover, because the performances are unique, there is no assurance that those who miss out on one show will choose or be able to attend another. This variability in venue size across different locations requires us to continually adapt our pricing strategy. With more adaptive prices, we can maximize both attendance and revenue while accommodating unpredictable changes in seat availability.
\end{example}
\begin{example}[Fruit Retails]\label{example_fruit}
    Sweetsop (\emph{Annona squamosa}, or so-called ``sugar apple'') is a  particularly-perishable tropical fruit, typically lasting only 2 to 4 days \citep{crane2005sugar}. Suppose we manage a local fruit shop and have partnered with a nearby farm for the supply of sweetsops during the harvest season. Due to their perishable nature, we receive sweetsops as soon as they are ripe and picked from the farm every day. This irregular supply means that some days we might receive a large quantity while getting very few on other days. We must quickly sell these fruits before they spoil, yet managing the price becomes challenging. If we exhaust our inventory ahead of time, customers will turn to other fruit shops for purchase instead of waiting for our next restock.  
\end{example}
Products in the two instances above have the following properties:
\begin{enumerate}
    \item Inventory levels are determined by natural factors, and are arbitrarily given for different individual time periods.
    \item Products are perishable and only salable within a single time period.
\end{enumerate}

\subsection{Problem Overview}
\label{subsec:problem_overview}
In this work, we study a dynamic pricing problem where the products possess these properties. The problem model is defined as follows. At each time $t=1,2,\ldots, T$, we firstly propose a price $p_t$, and then a price-dependent \emph{potential demand} occurs as $d_t$. However, we might have no access to $d_t$ as it is censored by an \emph{adversarial} inventory level $\gamma_t$. Instead, we observe a censored demand $D_t = \min\{\gamma_t, d_t\}$ and receive the revenue $r_t$ as a reward at $t$. Our goal is to approach the optimal price $p_t^*$ at every time $t$, thereby maximizing the cumulative revenue.

\fbox{\parbox{0.95\textwidth}{Dynamic pricing with adversarial inventory constraint. For $t=1,2,...,T:$
		\noindent
		\begin{enumerate}[leftmargin=*,align=left]
			\setlength{\itemsep}{0pt}
                \item The seller (we) receives $\gamma_t$ identical products.
			\item The seller proposes a price $p_t\geq0$.
                \item The customers generate an invisible potential demand $d_t\geq 0$, dependent on $p_t$.
			\item The market reveals an inventory-censored demand $D_t=\min\{\gamma_t, d_t\}$.
			\item The seller gets a reward $r_t = p_t\cdot D_t$.
                \item All unsold products perish before $t+1$.
		\end{enumerate}
	}
}

\bigskip
\paragraph{The notion of ``adversarial'' inventory.} We characterize this problem as having adversarially-chosen inventory levels, though our setting differs from classical adversarial bandits or online convex optimization. Here, an adversary may pre-commit to an arbitrary inventory sequence $\{\gamma_t\}_{t=1}^T$, while the underlying demand noise remains stochastic---a hybrid model that combines adversarial contexts with stochastic outcomes, following the convention of ``adversarial features'' in online learning \citep{cohen2020feature_journal, liu2021optimal}. Since inventory $\gamma_t$ is revealed at the beginning of round $t$, we can define per-round optimal prices $p_t^*$. This leads to a stronger regret benchmark: we compete against this sequence of optimal actions $\{p_t^*\}_{t=1}^T$ rather than the best fixed decision in hindsight typical of standard adversarial settings. This distinction is crucial as the inventory levels create a non-stationary optimization landscape where optimal prices vary dramatically across periods. 

\subsection{Summary of Contributions}
\label{subsec:summary_contribution}
We consider the problem setting shown above and assume that the potential demand $d_t=a-bp_t+N_t$ is \emph{linear} and \emph{noisy}. Here $a, b\in\R^+$ are fixed unknown parameters and $N_t$ is an \emph{unknown} and i.i.d. (independently and identically distributed) noise with zero mean. Under this premise, the key to obtaining the optimal price is to accurately learn the expected reward function $r(p)$, which is equivalent to learning the linear parameters $[a,b]$ and the noise distribution. We are faced with three principal challenges:

\begin{enumerate}
    \item The absence of unbiased observations of the potential demand or its derivatives with respect to $p$, which prevents us from estimating $[a,b]$ directly.
    \item The dependence of the optimal prices on the noise distribution, which is assumed to be unknown and partially censored. 
    \item The arbitrariness of the inventory levels, leading to non-stationary and highly-differentiated optimal prices $\{p_t^*\}$ over time.
\end{enumerate}

In this paper, we introduce an algorithm that employs innovative techniques to resolve the aforementioned challenges. Firstly, we devise a pure-exploration phase that bypasses the censoring effect and obtains an unbiased estimator of $\frac1{b}$ (which leads to $\hat b$ and $\hat a$ as a consequence). Secondly, we maintain estimates of the noise CDF $F(x)$ and $\int F(x)dx$ over a series of discrete $x$'s, as well as the confidence bounds of each estimate. Thirdly, we design an \emph{optimistic} strategy, \textbf{C20CB} as ``Closest-To-Zero Confidence Bound'', that proposes the price $p_t$ whose reward derivative $r'_t(p_t)$ is probably $0$ or closest to $0$ among a set of discretized prices. As we keep updating the estimates of $r'_t(\cdot)$ with shrinking error bar, we asymptotically approach the optimal price $p^*_t$ since $r'_t(p_t^*)=0$ for any $t=1,2,\ldots, T$.

\paragraph{Novelty.} To the best of our knowledge, we are the first to study the online dynamic pricing problem under \emph{adversarial} inventory levels. Our C20CB algorithm attains an \emph{optimal} $\tilde{O}(\sqrt{T})$ regret guarantee with high probability. The methodologies we develop are crucial to our algorithmic design, and are potentially advancing a variety of online decision-making scenarios with censored feedback.

\subsection{Paper Structure}
\label{subsec:paper_structure}
The rest of this paper is organized as follows. We discuss related works in \Cref{sec:related_works}, and then describe the problem setting in \Cref{sec:preliminaries}. We propose our main algorithm C20CB in \Cref{sec:algorithm} and analyze its regret guarantee in \Cref{sec:regret_analysis}. We further discuss potential extensions in \Cref{sec:discussion}, followed by a brief conclusion in \Cref{sec:conclusion}.

\section{Related Works}
\label{sec:related_works}
Here we discuss the closest related works on dynamic pricing, inventory constraints, and network revenue management. For a broader introduction of related literature, please refer to \Cref{app:subsec:related_works}.

\paragraph{Data-driven dynamic pricing.}
Dynamic pricing for identical products is a well-established research area, starting with \citet{kleinberg2003value} and continuing through seminal works by \citet{besbes2009dynamic, broder2012dynamic, wang2014close, wang2021multimodal}. The standard approach involves learning a demand curve from price-sensitive demand arriving in real-time, aiming to approximate the optimal price. \citet{kleinberg2003value} provided algorithms with regret bounds of $O(T^{\frac23})$ and $O(\sqrt{T})$ for arbitrary and infinitely smooth demand curves, respectively. \citet{wang2021multimodal} refined this further, offering an $O(T^{\frac{k+1}{2k+1}})$ regret for $k$-times continuously differentiable demand curves. This line of inquiry is also intricately linked to the multi-armed bandit problems \citep{lai1985asymptotically, auer2002nonstochastic} and continuum-armed bandits \citep{kleinberg2004nearly}, where each action taken reveals a reward without insight into the foregone rewards of other actions.
\paragraph{Pricing with inventory concerns.}
Dynamic pricing problems begin to incorporate inventory constraints with \citet{besbes2009dynamic}, which assumed a fixed initial stock available at the start of the selling period. They introduced near-optimal algorithms for both parametric and non-parametric demand distributions, operating under the assumption that the inventory is non-replenishable and non-perishable. \citet{wang2014close} adopted a comparable framework but allowed customers arrivals to follow a Poisson process. In these earlier works, the actual demand is fully disclosed until the inventory is depleted. Subsequent research allows inventory replenishment, with the seller's decisions encompassing both pricing and restocking at each time interval. \citet{chen2019coordinating} proposed a demand model subject to additive / multiple noise and developed a policy that achieved $O(\sqrt{T})$ regret. More recent studies \citet{chen2020data, keskin2022data} explored the pricing of perishable goods where the unsold inventory will expire. However, the uncensored demand is observable as assumed in both works. Specifically, \citet{chen2020data} allowed recouping backlogged demand, albeit at a cost, and introduced an algorithm with optimal regret. \citet{keskin2022data} focused on the cases where both fulfilled demands and lost sales were observable.

\citet{chenbx2021nonparametric} and their subsequent work, \citet{chen2023optimal}, are the closest works to ours as they adopt similar problem settings: In their works, the demand is \emph{censored} by the inventory level and any leftover inventory or lost sales disappear at the end of each period. With the assumption of concave reward functions and the restriction of at most $m$ price changes, \citet{chenbx2021nonparametric} proposed MLE-based algorithms that attain a regret of $\tilde{O}(T^{\frac1{m+1}})$ in the well-separated case and $\tilde{O}(T^{\frac12+\epsilon})$ for some  $\epsilon=o(1)$ as $T\rightarrow\infty$ in the general case. Under similar assumptions (except infinite-order smoothness), \citet{chen2023optimal} developed a reward-difference estimator, with which they not only enhanced the prior result for concave reward functions to $\tilde{O}(\sqrt{T})$ but also obtained a general $\tilde{O}(T^{2/3})$ regret for non-concave reward functions. Our problem model mirrors their difficulty, as we also lack access to both the uncensored demand and its gradient. However, they allowed the sellers to determine inventory levels with sufficient flexibility, hence better balancing the information revealed by the censored demand and the reward from (price, inventory) decisions. On the other hand, we assume that the inventory level at each time period is provided \emph{adversarially} by nature, which could impede us from learning the optimal price in the worst-case scenarios. Furthermore, due to the non-stationarity of inventory levels in our setting, the optimal price $p_t^*$ deviates over time. Given this, the search-based methods adopted in \citet{chen2023optimal} are no longer applicable to our problem.


\paragraph{Network Revenue Management (NRM).} NRM \citep{talluri2006theory} studies pricing and allocation of shared resources in a network. In the settings of \citet{besbes2012blind} and subsequent works \citep{simchi2019blind, miao2024demand}, marginal observations for each product can induce adversarial supply due to cross-product resource occupation. We also consider adversarial inventories, but focus on addressing demand censoring and reducing regret via online learning. \citet{perakis2010robust} likewise used regret and considered censoring in demand data, but their minimax/maximin regret definitions differ from ours, and censoring is mainly used in empirical validation rather than for theoretical solutions. Representative NRM works include \citet{gallego1994optimal} on dynamic pricing with stochastic demand, establishing structural monotonicity and asymptotic optimality of simple policies; \citet{talluri1998analysis} showing bid-price controls are near-optimal in large capacities but not strictly optimal; and \citet{meissner2012network} addressing overbooking with product-specific no-shows via a randomized LP.

\section{Problem Setup}
\label{sec:preliminaries}

We have defined the problem setting in \Cref{subsec:problem_overview}. To further clarify the scope of our methodology, we make the following very first assumption before introducing further concepts. 

\begin{assumption}[Linear Demand]
    \label{def:demand}
    Assume the \emph{potential demand} $d_t=d_t(p):=a-bp+N_t$ is \textbf{linear} and \textbf{noisy}. Denote $d(p):=a-bp$ as the \emph{expected potential demand function}. Denote $D_t(p):=\min\{\gamma_t, d_t(p)\}$ as the \emph{censored demand} function.
\end{assumption}
A linear demand model has been widely used in pricing literature, including \citet{lobo2003pricing,van2012models,broder2012dynamic,cohen2021simple}. We highlight two primary reasons: (1) According to \citet{besbes2015surprising}, linear demand allows the prices to converge to the true optimum even under mismatching scenarios. (2) As shown in \citet{keskin2014dynamic}, the same analysis can be applied to Generalized Linear Model (GLM), which has enhanced capability of capturing the real-world demands (see e.g. \citet{bu2022context,wang2021dynamic}).

\subsection{Definitions}
\label{subsec:definitions}
Here we define some key quantities that are involved in the algorithm design and analysis. Firstly, we define distributional functions of the noise $N_t$.
\begin{definition}[Distributional Functions]
    \label{def:cdf_pdf}
    For $N_t$ as the demand noise, denote $F(x)$ as its \emph{cumulative distribution function} (CDF), $x\in\R$. Also, denote the following $G(x)$ as the \emph{integrated CDF}:
    \begin{equation}
        \label{equ:gdf}
        G(x):=\int_{-\infty}^x F(\omega)d\omega, x\in\R
    \end{equation}
    
\end{definition}
We will make more assumptions on the noise distribution later. Notice that we do not assume the existence of PDF for $N_t$. However, if there exists its PDF in specific cases, we will adopt $f(x)$ as a notation. Then, we define the revenue function and the regret.

\begin{definition}[Revenue Function]
    \label{def:revenue}
    Denote $r_t(p)$ as the expected revenue function of price $p$, satisfying
    \begin{equation}
        \label{equ:revenue_function}
        r_t(p) := p\cdot\E_{N_t}[D_t(p)|\gamma_t], p\geq 0.
    \end{equation}
    Also, denote $p_t^*:=\argmax_{p}r_t(p)$ as the optimal price at time $t$.
\end{definition}

\begin{definition}[Regret]
    \label{def:regret}
    Denote
    \begin{equation}
        \label{equ:regret_def}
        Regret:= \sum_{t=1}^T r_t(p_t^*)-r_t(p_t)
    \end{equation}
    as the \emph{cumulative regret} (or \emph{regret}) of the price sequence $\{p_t\}_{t=1}^T$.
\end{definition}
The definition of regret inherits from the tradition of online learning, capturing the \emph{performance difference} between the algorithm-in-use and the best benchmark that an omniscient oracle can achieve (which knows everything except the realization of noise).

\subsection{Assumptions}
\label{subsec:assumptions}
Firstly, we assume boundaries for parameters and price.

\begin{assumption}[Boundedness]
    \label{assumption:boundedness}
    There exist \emph{known} \emph{finite} constants $a_{\max}, b_{\min}, b_{ \max}, \gamma_{\min},$ and $ c >0$ such that $0<a\leq a_{\max}$, $0<b_{\min}\leq b\leq b_{\max}$, $\gamma_t\geq\gamma_{\min}$, $N_t\in[-c, c]$. Also, we restrict the proposed price $p_t$ at any $t=1,2,\ldots, T$ satisfies $0\leq p_t\leq p_{\max}$ with a \emph{known} \emph{finite} constant $p_{\max}>0$.
\end{assumption}

The assumption of boundedness on $a, b$ and price is natural as it defines the scope of instances. We justify the assumption of noise with bounded support $[-c, c]$ from two aspects: (1) The upper bound exists without loss of generality due to the existence of inventory-censoring effect. (2) The lower bound exists as we avoid negative demand. Secondly, we make assumptions on the noise distribution.
\begin{assumption}[Noise Distribution]
    \label{assumption:noise}
    Each $N_t$ is drawn from an \emph{unknown} independent and identical distribution (i.i.d.) satisfying $\E[N_t] = 0$. The CDF $F(x)$ is $L_F$-\emph{Lipschitz}. Also, according to \Cref{assumption:boundedness}, we have $F(-c)=0, F(c)=1$.
\end{assumption}

Thirdly, we make assumptions on the inequality relationships among parameters:

\begin{assumption}[Inequalities of Parameters]
    \label{assumption:inequality}
    The parameters and constants involved in the problem setting satisfy the following conditions:
\begin{table}[h]
    \begin{tabular}{|l|l|l|ll}
    \cline{1-3}
      & Assumption & Explanation                                                                                                                                                                  &  &  \\ \cline{1-3}
    (1)\ & $a-c > \gamma_t, \forall t\in[T]$        & \begin{tabular}[c]{@{}l@{}}The demand at $p=0$ is completely censored by any inventory level,\\ since customers will rush to buy until completely out-of-stock.\end{tabular} &  &  \\ \cline{1-3}
    (2)\ & $\gamma_t > 2c, \forall t\in[T]$   & \begin{tabular}[c]{@{}l@{}}Inventory level should exceed the width of the noise support.\\ Otherwise we can reshape the noise by capping $N_t$ at $\gamma_t-a+bp_{\max}$.\end{tabular}                            &  &  \\ \cline{1-3}
    (3)\ & $a-b p_{\max} - c > 0$       & Demands must be positive.                                                                                                                                                    &  &  \\ \cline{1-3}
    (4)\ & $\gamma_{\min} > a_{\max} -b_{\min} p_{\max} + c$ & \begin{tabular}[c]{@{}l@{}}Demands at $p_t=p_{\max}$ must be uncensored.\\ We denote $\gamma_0:=a_{\max} -b_{\min} p_{\max} + c$ for further use.\end{tabular}                                                         &  &  \\ \cline{1-3}
    (5)\ & $p_{\max}\geq \frac{a}{2b}$       & Optimal price must be included in $[0, p_{\max}]$ without loss of generality.                                                                                                  &  &  \\ \cline{1-3}
    \end{tabular}
\end{table}
\end{assumption}

Each assumption in \Cref{assumption:inequality} is justified by an explanation followed. We make Item (3) to avoid simultaneously upper-and lower-censoring effect, which is beyond the scope of this work and considered as an extension. It is worth mentioning that for those boundedness parameters (e.g. $a_{\max}, b_{\max}, c, \gamma_{\min}$), we have to know their exact values, while we only require the existence of $a, b$ without knowing them. Finally, for the benefit of regret analysis, we assume that the time horizon $T$ is sufficiently large, such that its polynomial will not confound any constant or coefficient.

\begin{assumption}[Large $T$]
    \label{assumption:large_T}
    For any constant $r=O(1)$, time horizon $T$ is larger than any polynomial of parameters, i.e. $T>\Omega\left((\frac{a_{\max}b_{\max}cp_{\max}}{b_{\min}\gamma_{\min}})^r\right)$.
\end{assumption}

\section{Algorithm Design}
\label{sec:algorithm}
In this section, we present our core algorithm,\textbf{C20CB} (see \Cref{algo:C20CB}), which stands for a \emph{Closest-To-Zero Confidence Bound} strategy that proposes asymptotically optimal prices over differentiated inventory levels and censoring effects.  

\begin{algorithm}[htbp]
    \caption{C20CB: Closest-To-Zero Confidence Bound (Main Algorithm)}
    \label{algo:C20CB}
    \begin{algorithmic}[1]
        \STATE {\bfseries Input}: {Parameters $a_{\max}, b_{\min}, b_{\max}, p_{\max}, c$ and self-derived quantities $\tau, \gamma_0, C_a, C_b, C_F, C_G, C_N, C_{\tau}$.}
        \STATE \textbf{//STAGE 1: Pure Exploration for $\tau$ times.}
        \STATE Estimate $(\hat b, \hat a) \leftarrow \mathit{PureExp}(\tau, p_{\max}, \gamma_0)$ according to \textbf{\Cref{algo:pure_exploration_phase}}.
        \STATE \textbf{//STAGE 2: Optimistic Acting}
        \STATE Define $\Delta := (C_a + C_b\cdot p_{\max})\cdot\frac1{\sqrt{\tau}}, M:=\lfloor\frac{c}{2\Delta}\rfloor$ and $w_k:=2k\Delta, k=-M, -M+1, \ldots, M$.
        \STATE Initialize $(F_k, N_k, G_k, \Delta_k)\leftarrow\mathit{ConfBoundInit}(\tau, M, \hat a, \hat b, c, b_{\max}, p_{\max}, C_b, C_F, C_G)$ for $k=-M, \ldots, M$, according to \textbf{\Cref{algo:conf_bound_init}}.
        \FOR{$t=1,2,\ldots, T-\tau-(2M+1)$}
            \IF{$\gamma_t\geq\frac{\hat a + C_a\cdot\frac1{\sqrt{\tau}}}2 + c$}
                \STATE Propose $p_t \leftarrow \frac{\hat a}{2\hat b}$ and continue to $t+1$ (without recording feedback).
            \ELSE
                \STATE Get $(p_t, k_t) \leftarrow \mathit{OptPrice}(M, \gamma_t, \{(F_k, N_k, G_k,\Delta_k)\}_{k=-M}^M, c, \gamma_0)$ according to \textbf{\Cref{algo:opt_price}}.
                \STATE Propose $p_t$ as the price, and observe $D_t$ and $\ind_t:=\ind[D_t\leq \gamma_t]$.
                \IF {$k_t > -\infty$}
                    \STATE Update
                    \begin{equation}
                        \label{eq:updates_DFGN}
                        F_{k_t}\leftarrow\frac{N_{k_t} F_{k_t} + \ind_t}{N_{k_t}+1},\ G_{k_t}\leftarrow\frac{N_{k_t} G_{k_t} + D_t - \gamma_t + c}{N_{k_t}+1},\ \Delta_{k_t}\leftarrow \frac{C_N}{\sqrt{N_{k_t}+1}} + \frac{C_{\tau}}{\sqrt{\tau}}
                    \end{equation}
                    \STATE Update $N_{k_t}\leftarrow N_{k_t} + 1$.
                \ENDIF
            \ENDIF
        \ENDFOR
    \end{algorithmic}
\end{algorithm}

\subsection{Algorithm Design Overview}
\label{subsec:algorithm_design}

Our algorithm has two stages:
\begin{enumerate}
    \item { \textbf{STAGE 1: Exploration}}: During the first $\tau=\sqrt{T}$ rounds, the seller (we) proposes uniformly random prices in the range of $[0, p_{\max}]$. By the end of STAGE 1, we obtain $\hat{a}$ and $\hat{b}$ as plug-in estimators of $a$ and $b$ in the following stage.
    \item { \textbf{STAGE 2: Optimistic Decision}}: We estimate the derivatives of the revenue function at discretized prices $\{p_{k,t}\}$'s. For each $p_{k,t}$, we not only estimate $r'_t(p_{k,t})$ but also maintain an error bar of that estimate. At each time $t$, we propose the price whose corresponding error bar covers $0$ or closest to $0$ if no covering exists.
\end{enumerate}

\Cref{algo:C20CB} exhibits several advantageous properties. It is suitable for processing streaming data as the constructions of $\hat a, \hat b, \hat r'_t(\cdot)$ are updated \emph{incrementally} with each new observation (including $e_{i,t}, D_t, \ind_t$) without the need of revisiting any historical data. Additionally, it consumes $\tilde{O}(T^{\frac54})$ time complexity and $O(T^{\frac14})$ extra space, which are plausible for large $T$. A potential risk of computation might arise on the calculation of $\hat b$, where $\sum_{t=1}^{\tau}e_{1,t}-e_{2,t}$ can be 0 with a small but nonzero probability. Although this event does not undermine the high-probability regret guarantee, it might still be harmful to the computational system for numerical experiments. To mitigate this incident in practice, we may either extend STAGE 1 until one non-zero $e_{1,t}-e_{2,t}=1$ is observed, or restart STAGE 1 at $t=\tau$.


\begin{algorithm}[t]
    \caption{\textit{PureExp}: Pure Exploration}
    \label{algo:pure_exploration_phase}
    \begin{algorithmic}[1]
        \FOR{ $t=1,2,\ldots, \tau$}
            \STATE Sample and propose a price $p_t\sim U[0, p_{\max}]$ uniformly at random.
            \STATE Observe demand $D_t$, and indicators $e_{i,t}$ as defined in \Cref{eq:def_e_i_t} for $i=1,2,3$.
        \ENDFOR
        \STATE Estimate
            \begin{equation}
                \label{eq:a_hat_and_b_hat}
                \begin{aligned}
                    \hat b = & \frac1{4p_{\max}\cdot\frac1{\tau}\sum_{t=1}^\tau\frac{e_{1,t}-e_{2,t}}{\gamma_t - \gamma_0}}, \quad 
                    \hat a =  \frac1{\tau}\sum_{t=1}^{\tau}( \hat b p_{\max}e_{3,t} + \frac{3\gamma_t+\gamma_0}4).
                \end{aligned}
            \end{equation}
    \end{algorithmic}
\end{algorithm}

\begin{algorithm}[t]
    \caption{\textit{ConfBoundInit}: Confidence Bound Initialization}
    \label{algo:conf_bound_init}
    \begin{algorithmic}[1]
        \FOR{$t=1,2,\ldots, 2M+1$}
            \STATE Let $k_t:=-M-1+t$ and propose $p_t=\frac{w_{k_t}-(\gamma_t-\hat a)}{\hat b}$.
            \STATE Observe $D_t$ and $\ind_t:=\ind[D_t<\gamma_t]$.
            \STATE Initialize $F_{k_t}\leftarrow \ind_t, N_{k_t}\leftarrow 1, G_{k_t}\leftarrow D_t-\gamma_t + c, \Delta_{k_t}\leftarrow C_F\cdot b_{\max}p_{\max} + C_G + C_b\cdot\frac1{\sqrt{\tau}}$.
        \ENDFOR
    \end{algorithmic}
\end{algorithm}

\begin{algorithm}[t]
    \caption{\textit{OptPrice}: Select Optimal Price}
    \label{algo:opt_price}
    \begin{algorithmic}[1]
        \STATE Initialize $k_t \leftarrow M$ as the index of arm to inspect, $\rho_t\leftarrow +\infty$ as the smallest absolute value of the derivative estimates we have observed in Time Period $t$ so far.
        \FOR{$k=M, M-1, \ldots, -M+1, -M$}
            \STATE Denote $p_{k, t}:=\frac{w_k-(\gamma_t-\hat a)}{\hat b}$ and $\hat r_{k,t}:=\gamma_0 - c + G_k - \hat b \cdot p_{k,t}\cdot F_k$.
            \IF {$\hat r_{k, t} - \Delta_{k}\leq 0 \leq \hat r_{k, t} + \Delta_{k}$}
                \STATE Update $k_t\leftarrow k, \rho_t\leftarrow 0$, and Break.
            \ENDIF
            \STATE Let $\rho_{k, t}:=\min\{|\hat r_{k, t} - \Delta_{k}|, |\hat r_{k, t} + \Delta_{k}|\}$ as the smallest absolute derivative estimate of arm $k$.
            \IF {$\rho_{k, t} < \rho_t$}
                \STATE Update $\rho_t\leftarrow\rho_{k, t}$ and $k_t\leftarrow k$.
            \ENDIF
        \ENDFOR
        \IF{$\hat r_{k, t}-\Delta_k>0, \forall k=-M, -M+1, \ldots, M-1, M$}
            \STATE Output $p_t\leftarrow\frac{\hat a}{2 \hat b}$ and $k_t \leftarrow -\infty.$
        \ELSE
            \STATE Output $p_t \leftarrow p_{k_t, t}$ and $k_t$.
        \ENDIF
    \end{algorithmic}
\end{algorithm}
\begin{figure}[t]
    \centering
    \includegraphics[width=0.9\linewidth]{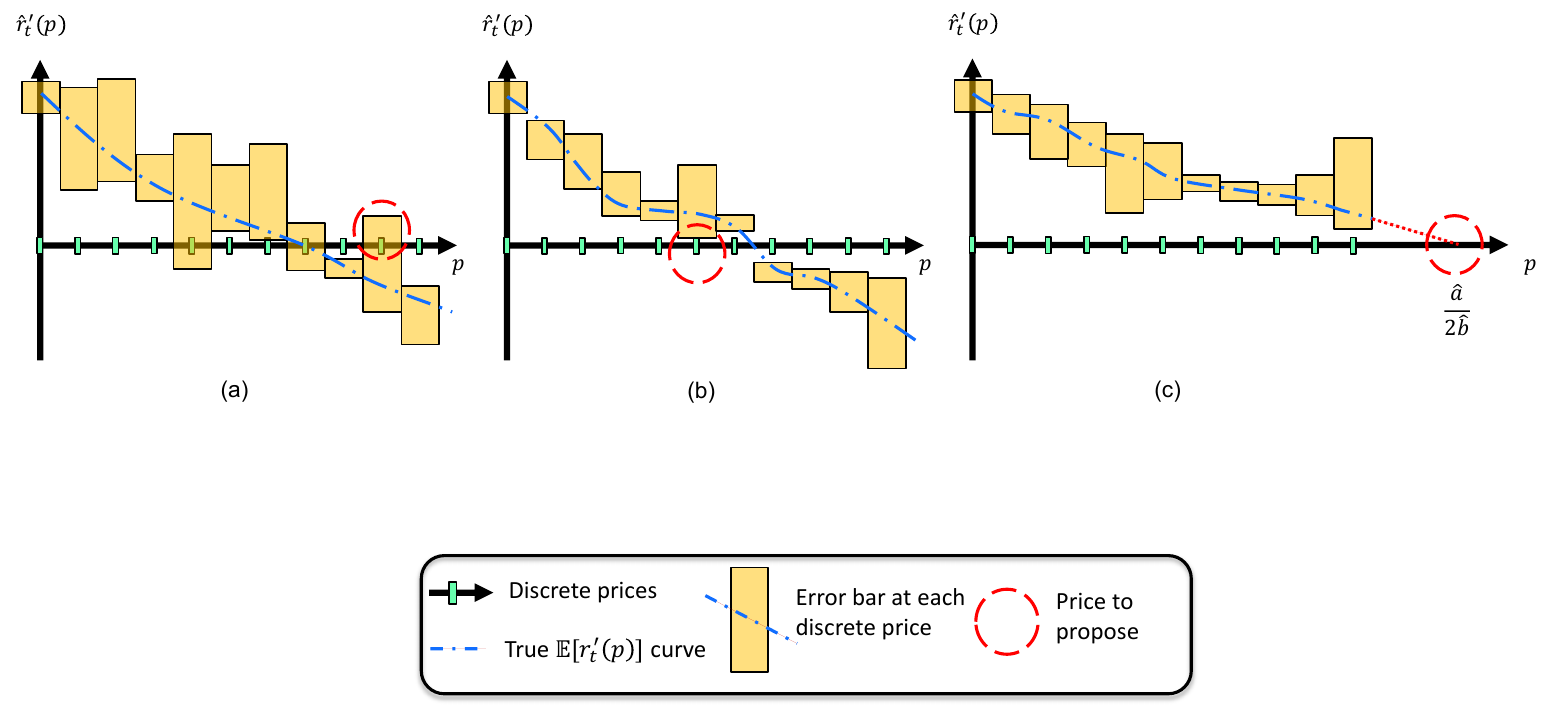}
    \caption{The price C20CB proposes based on confidence bounds of $\hat r_{k,t}$: (a) If there exist prices whose error bar contain $0$, then we propose the largest price among them. (b) If no error bar contains $0$ but there does exist at least one below $0$, we propose the price whose corresponding error bar is closest to $0$. (c) If all error bars are above $0$, we propose $p_t=\frac{\hat a}{2\hat b}$.}
    \label{fig:illustration_c20cb}
\end{figure}

\subsection{Pure-Exploration to Estimate Parameters from Biased Observations}
\label{subsec:pure_exploration}
As shown in \Cref{algo:pure_exploration_phase}, we incorporate a uniform exploration phase to estimate $a$ and $b$, bypassing the obstacle caused by demand censoring. This approach is supported by the following insight: When $Y$ is a uniformly distributed random variable within a closed interval $[L,R]$, and $X$ is another random variable, independent to $Y$ and also distributed within $[L,R]$, we have:
\begin{equation}\label{equ:total_expectation}
    \E[\ind[Y\geq X]]=\Pr[Y\geq X] = \E[\Pr[Y\geq X| X]] = \E[\frac{X-L}{R-L}] = \frac{\E[X]-L}{R-L}.
\end{equation}
Here the second step uses the Law of Total Expectation. \Cref{equ:total_expectation} indicates that we can derive an unbiased estimator of $\E[X]$ through $\ind[Y\geq X]$ even in the absence of any direct observation of $X$. Looking back to our algorithm, denote $\gamma_{i, t}:=\frac{i\gamma_t + (4-i)\gamma_0}4, i=1,2,3$ as the three quarter points of $\gamma_t$ for $t=1,2,\ldots, T$, and define
\begin{equation}
    \label{eq:def_e_i_t}
    \begin{aligned}
        e_{i,t}:=\ind[D_t\geq\gamma_{i,t}], i=1,2,3, t=1,2,\ldots, T.
    \end{aligned}
\end{equation}
When $p_t\sim U[0, p_{\max}]$, we have 
\begin{equation}
\label{equ:e_i_expectation_demo}
\begin{aligned}
    \E[e_{i,t}] & = \E[\ind[a-bp_t+N_t\geq \gamma_{i, t}]] = \E[\E[N_t\geq \gamma_{i, t}-a+bp_t|N_t]]=\E[\frac{N_t-\gamma_{i, t}+a}{bp_{\max}}]=\frac{a-\gamma_{i,t}}{bp_{\max}}.
\end{aligned}
\end{equation}
The last equality comes from $\E[N_t]=0$. By deploying different $\gamma_{i, t}$ at $i=1,2,3$, we can estimate $a$ and $b$ through the observations of $e_{i,t}$ according to \Cref{algo:pure_exploration_phase}, effectively circumventing the censoring effect. A similar technique has been used by \citet{fan2021policy} to construct an unbiased estimator of \emph{valuation} instead of the demand, as we are concerned. However, their application of uniform exploration might be sub-optimal as they adopt an \emph{exploration-then-exploitation} design. In contrast, our algorithm uses this uniform exploration merely as a \emph{trigger} of further learning. Our tight regret bound indicates that uniform exploration can still contribute to an optimal algorithm for a broad range of online learning instances.

\subsection{Optimistic Strategy to Balance Derivatives Estimates v.s. Loss}
\label{subsec:optimistic_strategy}
With $\hat a$ and $\hat b$ established, we have an estimate of the underlying linear demand $d_t(p)=a-bp$. However, we are still unaware of the noise distribution, which is crucial for the current optimal price, since the inventory level $\gamma_t$ partially censors the noise.

In order to balance the learning of noise distribution versus the loss of proposing sub-optimal prices, we apply an \emph{optimistic} strategy in STAGE 2. Usually, an ``optimistic strategy'' chooses actions as if the best-case scenario within its current confidence region holds, thereby encouraging exploration of potentially rewarding options while still balancing risk. 

STAGE 2 involves the following components:
\begin{enumerate}[label=(\roman*)]
    \item We discretize the $[-c,c]$ domain of noise CDF $F(\cdot)$ and its integration $G(\cdot)$ into small intervals of length $2\Delta$, with $\Delta=O(\frac1{T^{1/4}})$. At the center of interval $k$ (which is $2k\Delta$) for each $k=-M, -M+1, \ldots, M$ (with $M:=\lfloor\frac{c}{2\Delta}\rfloor$), we maintain independent estimates of $F$ and $G$, including their expectations and high-confidence error bars.
    \item At each time $t$, we construct a set of discrete prices $\{p_{k,t}\}_{k=-M}^{M}$ such that the quantity $\gamma_t-\hat a+\hat bp_{k,t}$ matches the center of Interval $k$. Given this, we further construct an estimate of each $r'_t(p_{k,t})$ with plug-in estimators $\hat a, \hat b$ and discrete estimators of $F$ and $G$ for the specific Interval $k$. The estimate includes its expectation and error bar (see \Cref{algo:conf_bound_init}).
    \item Since the optimal price $p_t^*$ satisfies $r'_t(p_t^*)=0$, we identify the discrete price $p_{k,t}$ where the derivative estimate is ``possibly $0$'' or ``closest to $0$''. To make this, we design \Cref{algo:opt_price} as a component of C20CB and illustrate the process in \Cref{fig:illustration_c20cb}, which includes the following three cases:
    \begin{enumerate}[label=(\alph*)]
        \item If there exists some $p_{k,t}$ such that the corresponding error bar (of its derivative estimate) contains $0$, we propose the largest price satisfying this condition. 
        \item If there is no $p_{k,t}$ whose corresponding error bar contains $0$ (but there exists an error bar below $0$), we propose the price whose \emph{error bound} is closest to $0$. 
        \item If all error bars are above $0$, indicating that the reward function is monotonically increasing over the ``censoring area'', we propose $p_t=\frac{\hat a}{2\hat b}$ to exploit the non-censoring optimal price $\frac{a}{2b}$. In this case, we do not need to record any observations nor to update any parameter/estimate. 
        
    \end{enumerate}
    \item After proposing the price $p_{k_t, t}$ and observing feedback $D_t$ and $\ind_t$, we update the estimates of $F(\cdot)$ and $G(\cdot)$ for Interval $k_t$ in which $\gamma_t-\hat a+\hat bp_{k_t,t}$ exists.
\end{enumerate}

In a nutshell, we maintain estimates and error bars of $F(\cdot)$ and $G(\cdot)$ at discrete points $2k\Delta$, and map each $2k\Delta$ to a corresponding price $p_{k,t}$ once an inventory $\gamma_t$ occurs. Then we propose the price whose derivative estimate $\hat r_{k,t}\pm\Delta_k$ is \emph{closest-to-zero} among all $k$. Finally, we update the estimates with observations. 

Here we provide an intuition of the optimality: On the one hand, the width of the interval can tolerate the error of mapping from $2k\Delta$ to $p_{k,t}$, and the Lipschitzness of $F$ and $G$ ensures that our estimate within each small interval is roughly correct. On the other hand, we can show that the closest-to-zero derivative estimate implies a \emph{closest-to}-$p_t^*$ price according to some locally strong convexity in the neighborhood of $p_t^*$. As we have smoothness on the regret function, we suffer a quadratic loss at $(T\cdot\frac1{T^{1/4}})^2=O(\sqrt{T})$ cumulatively, which balances the loss of STAGE 1 that costs $O(\tau)=O(\sqrt{T})$ as well. For a rigorous regret analysis, we kindly refer the readers to \Cref{sec:regret_analysis}.

{
\paragraph{Technical Highlights.}
This work is the first to introduce \emph{optimism} on the derivatives and achieve optimal regret in an \emph{adversarial} online learning problem. In contrast, existing works either develop optimistic algorithms on the reward (or loss) function as the original UCB strategy \citep{lai1985asymptotically}, or instead use unbiased stochastic gradients and conduct first-order methods for online optimization \citep{hazan2019introduction}.}

\section{Regret Analysis}
\label{sec:regret_analysis}
\vspace{-0.5em}
In this section, we analyze the cumulative regret of our algorithm and show a $\tilde{O}(\sqrt{T})$ regret guarantee with high probability. Here we only display key lemmas and proof sketches, and we leave all of the proof details to \Cref{app:sec:proof_detail}.

We first state our main theorem.
\begin{theorem}[Regret]
    \label{thm:regret}
    Let $\tau={\sqrt{T}}$ in \Cref{algo:C20CB}. For any adversarial $\{\gamma_t\}_{t=1}^T$ input sequence, C20CB suffers at most $\tilde{O}\left(\sqrt{T}\cdot\log\frac{T}{\delta}\right)$ regret, with probability $\Pr\geq 1-\delta$.
\end{theorem}

\begin{proof}
    \label{proof:thm_regret_sketch}
    In order to prove \Cref{thm:regret}, we have to show the following three components:
    \begin{enumerate}
        \item The reward function $r_t(p)$ is unimodal. Also, $r_t(p)$ is smooth at $p_t^*$, and is strongly concave on a neighborhood of $p_t^*$.
        \item The estimation error of $a$ and $b$ are bounded by $O(\frac1{T^{1/4}})$ at the end of STAGE 1.
        \item The price whose derivative estimate has the closest-to-zero confidence bound is asymptotically close to $p_t^*$.
    \end{enumerate}

In the following, we present each corresponding lemma regarding to the roadmap above.
\begin{lemma}[Revenue Function $r_t(p)$]
    \label{lemma:r_t_p}
    For the expected revenue function $r_t(p)$ defined in \cref{equ:revenue_function}, the following properties hold:
    \begin{enumerate}
        \item We have
        \begin{equation}
            \label{equ:r_p_and_derivatives}
            \begin{aligned}
                r_t(p) & = p(\gamma_t-c+G(c)-G(\gamma_t-a+bp))\\
                r'_t(p) & = \gamma_t-c+G(c)-G(\gamma_t-a+bp)-bp\cdot F(\gamma_t-a+bp).
            \end{aligned}
        \end{equation}
        \item There exists a constant $L_r>0$ such that $r'(p)$ is $L_r$-\emph{Lipschitz}.
        \item $r'(p)$ is monotonically \emph{non-increasing}.
        \item $r_t(p)$ is \emph{unimodal}: There exists a unique $p_t^*\in[0, \frac{a}{b}]$ such that $r'_t(p_t^*)=0$, and $r_t(p)$ monotonically increase in $[0, p_t^*]$ and decrease in $[p_t^*, \frac{a}{b}]$. Notice that $\frac{a}{b} > p_{\max}$ according to \Cref{assumption:inequality}.
        \item $r_t(p)$ is \emph{smooth} at $p_t^*$: There exists a constant $C_s >0$ such that $r_t(p_t^*)-r_t(p)\leq C_s(p_t^*-p)^2, \forall p\in[0, p_{\max}]$.
        \item $r_t(p)$ is \emph{locally strongly concave}: There exist $\epsilon_t>0$ and $C_{\epsilon}>0$ such that $\forall p_1, p_2\in[p_t^*-\epsilon_t, p_t^*+\epsilon_t]$ we have $|r'_t(p_1)-r'_t(p_2)|\geq C_{\epsilon}\cdot|p_1-p_2|$.
        \item There exists a constant $C_v>0$ such that for any $t\in[T]$ and $p\in(p_t^*-\epsilon_t, p_t^*+\epsilon_t)$, we have $|r_t(p_t^*)-r_t(p)|\leq C_v\cdot(r'_t(p))^2$.
    \end{enumerate}
\end{lemma}
\begin{proof}[sketch]
    Property 1 is from integration by parts. Property 2 is proved by the Lipschitzness of $F(x)$. Property 3 is from the monotonicity of $F(x)$. Property 4 can be proved by two steps: (4.1) The existence of $p_t^*\in[0,\frac{a}{b}]$ by $r_t'(0)>0$ and $r_t'(\frac{a}{b})<0$; (4.2) The uniqueness of $p_t^*$ by contradiction. Properties 5 and 6 are mainly from the Lipschitzness of $r'_t(p)$. Property 7 comes from $r'_t(p_t^*) = 0$ and the strong concavity (Property 6).  Please kindly check \Cref{app:subsec:proof_lemma_r_t_p} as a detailed proof of \Cref{lemma:r_t_p}.  \qed
\end{proof}
The properties of $r_t(p)$ and $r'_t(p)$ enable us to upper bound the cost of estimation error and decision bias. In the following, we propose a lemma that serves as a milestone of estimation error upper bounds.

\begin{lemma}[Estimation Error of $a$ and $b$]
    \label{lemma:a_b_estimation_error}
    For any $\eta>0, \delta>0$, with probability $\Pr\geq 1-2\eta\delta$, we have
    \begin{equation}
        \label{equ:error_b_hat_and_a_hat}
        \begin{aligned}
            |\hat b - b| \leq& C_b\cdot\frac1{\sqrt{\tau}}, \quad
            |\hat a - a| \leq C_a\cdot\frac1{\sqrt\tau},
        \end{aligned}
    \end{equation}
    where $C_a:=p_{\max}(C_b + b_{\max}\cdot\sqrt{\frac12\log\frac{2}{\eta\delta}})$ and $C_b:=\frac{8b_{\max}^2}{\gamma_{\min}-\gamma_0}\cdot\sqrt{\frac12\log\frac{2}{\eta\delta}}$.
\end{lemma}
\begin{proof}[sketch]
    The key observation to prove this lemma lies in the expectation of each $e_{i,t}$ that indicates whether the demand exceeds certain level under uniformly distributed prices. According to Law of Total Expectation, for any $\gamma \in[\gamma_{\min}, \gamma_t]$ at any time $t\leq\tau$ in STAGE 1, we have
    \begin{equation}
        \begin{aligned}
            \E[\ind[D_t\geq \gamma]]=&\E_{N_t}[\E_{p_t}[a-bp_t+N_t\geq \gamma|N_t]]\\
            =&\E_{N_t}[\E_{p_t}[p_t\leq\frac{a-\gamma+N_t}{b}|N_t]]\\
            =&\E_{N_t}[\frac{a-\gamma+N_t}{b\cdot p_{\max}}]=\frac{a-\gamma}{bp_{\max}}.
        \end{aligned}
    \end{equation}
    With this property, we may construct method-of-moment estimates from $e_{i,t}$ in STAGE 1, which eliminate the influence of noise distribution and achieve an unbiased estimator of $\frac1b$ (and therefore $\hat b$ asymptotically). 
    With $\hat b$ serving as a plug-in estimator, we later achieve $\hat a$. We obtain the error bounds by applying Hoeffding's Inequalities. \qed 
\end{proof}

We defer the detailed proof of \Cref{lemma:a_b_estimation_error} to \Cref{app:subsec:proof_lemma_a_b_estimation_error}. With the help of \Cref{lemma:a_b_estimation_error}, we may upper bound the estimation error of $r'_t(p)$ at discrete prices. The error bound is displayed as the following lemma.

\begin{lemma}[Estimation Error of $r'_t(p_{k, t})$]
    \label{lemma:r_derivative_estimation_error}
    There exists constants $C_N>0, C_{\tau}>0$ such that for any $t\in[T], k\in\{-M, -M+1, \ldots, M\}$, with probability $\Pr\geq 1-6\eta\delta$ we have
    \begin{equation}
        \label{eq:derivative_estimation_error}
        |r'_t(p_{k,t})-\hat r_{k,t}|\leq C_N\cdot\frac1{\sqrt{N_k(t)}} + C_\tau\cdot\frac1{\sqrt{\tau}} = \Delta_k(t).
    \end{equation}
    Here $N_k(t)$ and $\Delta_k(t)$ denotes the value of $N_k\text{ and }\Delta_k$ at the beginning of time period $t$.
\end{lemma}
\begin{proof}[sketch]
    Denote $N_k(t), F_k(t)$ and $G_k(t)$ as the value of $N_k, F_k$ and $G_k$ at the beginning of time $t$. From Algorithm 1, we have $|\hat r_{k,t} - r'_t(p_{k,t})|\leq|G_k(t)-(G(c)-G(\gamma_t-a+bp_{k,t}))|+p_{k,t}|\hat bF_k-bF(\gamma_t-a+bp_{k,t})|$, and we may separately upper bound each of these two differences. 
    \begin{enumerate}
        \item For the term on $G$, we may split the quantity $|G_k(t)-(G(c)-G(\gamma_t-a+bp_{k,t}))|$ into three terms:
        \begin{enumerate}[label=(\roman*)]
            \item $|G_k(t) - \E[G_k(t)]|$ as concentration, which is bounded by Hoeffding's inequality.
            \item $|\E[G_k(t)] - (G(c) - G(2k\Delta))|$ as the $G_k$ deviation within each ``bin'' of estimation (whose center is $w_k = 2k\Delta$). By definition of $G_k(t)$, we have $\E[G_k(t)] = \frac1{N_k(t)}\sum_{s=1}^{t-1}\ind[k_s==k](G(c) - G(\gamma_s-a+bp_{k_s,s}))$, and therefore we have
            \begin{equation}
                \label{eq:lemma_3_proof_sketch}
                \begin{aligned}
                    &|\E[G_k(t)] - (G(c) - G(2k\Delta))|\\
                    \leq&\frac1{N_k(t)}\sum_{s<t: k_s==k}|G(\gamma_s-a+bp_{k_s,s})-G(2k\Delta)|\\
                    \leq& O(|\hat{a}-a| + |\hat b - b|).
                \end{aligned}
            \end{equation}
            The last line is by Lipschitzness of $G(x)$.
            \item $|G(2k\Delta) - G(\gamma_t - a + bp_{k,t})|$, which is similar to the last step of (ii) above.
        \end{enumerate}
        \item For the term on $F$, we may split the quantity $|\hat bF_k-bF(\gamma_t-a+bp_{k,t})|$ into four terms:
        \begin{enumerate}[label=(\roman*)]
            \item $|\hat b F_k(t) - b F_k(t)|\leq |\hat b - b|$ by the nature of $F_k(t)\leq 1$, and further bounded by \Cref{lemma:a_b_estimation_error}. 
            \item $|F_k(t) - \E[F_k(t)]|$, (iii) $|\E[F_k(t)] - F(2k\Delta)|$, and (iv) $|F(2k\Delta) - F(\gamma_t - a + bp_{k,t})|$, are bounded in the same way presented above for $G$ respectively.
        \end{enumerate}
        By plugging in the estimation error from \Cref{lemma:a_b_estimation_error}, we prove the present lemma. \qed
    \end{enumerate}
\end{proof}
Please refer to \Cref{app:subsec:proof_lemma_r_derivative_estimation_error} as a rigorous proof of \Cref{lemma:r_derivative_estimation_error}. Given this lemma, the derivatives of each discrete price $p_{k,t}$ is truthfully reflected by their corresponding error bound. Therefore, we intuitively see that the closest-to-zero confidence bound represents the closest-to-$p_t^*$ discrete price. We formulate this intuition as the following lemma:

\begin{lemma}[Closest-To-Zero Confidence To Performance]
    \label{lemma:c20_to_performance}
    Denote $\Delta_k(t)$ as the value of $\Delta_k$ at the beginning of period $t$. There exists two constants $N_0>0, N_1>0$ such that for any $t=1,2,\ldots, T$ in STAGE 2, either of the following events occurs with high probability.
    \begin{enumerate}
        \item When $\exists k\in\{-M, -M+1, \ldots, M\}$ such that the Number $k$ confidence bound satisfies $\hat{r}_{k,t}-\Delta_k(t)\leq 0 \leq \hat{r}_{k,t}+\Delta_k(t)$, and also $N_k(t)>N_0$, then we have $p_{k, t}\in[p_t^*-\epsilon_t, p_t^*+\epsilon_t]$. Furthermore, there exists constant $C_{in}$ such that $r_t(p_t^*)-r_t(p_{k,t})\leq C_{in}(\frac1{N_k(t)} + \frac1{\tau})$.
        \item When there exists no confidence bound that contains $0$, i.e. either $\hat{r}_{k,t}-\Delta_k(t)>0$ or $\hat{r}_{k,t}+\Delta_k(t)<0, \forall k\in\{-M, -M+1, \ldots, M-1, M\}$ (happens at least for one $k$), and also $N_k(t)>N_1$, then we have
        \begin{equation}
        \label{eq:lemma_4_event_2_no_confidence_bound_contains_0}
        \inf_{k}\min\{|\hat{r}_{k,t}-\Delta_k(t)|, |\hat{r}_{k,t}+\Delta_k(t)|\}\leq C_{\kappa}\cdot\frac1{\sqrt{\tau}},
        \end{equation}
        
        (where $C_{\kappa}=\frac{L_r(C_a + C_b\cdot p_{\max})}{2b_{\min}}$) and $p_{k, t}\in[p_t^*-\epsilon_t, p_t^*+\epsilon_t]$. Furthermore, there exists constant $C_{out}$ such that $r_t(p_t^*)-r_t(p_{k,t})\leq C_{out}(\frac1{N_k(t)} + \frac1{\tau})$. 
    \end{enumerate}
\end{lemma}
\begin{proof}[sketch]
    The intuition to prove \Cref{lemma:c20_to_performance} is twofold:
\begin{enumerate}
    \item When an error bar contains $0$, the true derivative of the corresponding price is close to $0$ within the distance of its error bound. By applying \Cref{lemma:r_t_p} Property (7), we may upper bound the performance loss with the square of its derivatives, which is further upper bounded by the square of error bound.
    \item When no error bar contains $0$, there exists an adjacent pair of prices whose error bars are separated by $y=0$. On the one hand, their derivatives difference is upper bounded due to the Lipschitzness of $r'_t(p)$. On the other hand, the same derivatives difference is lower bounded by the closest-to-zero confidence bound. Therefore, the gap between $y=0$ and the closest-to-zero confidence bound should be very small, and we still have a comparably small $|r'_t(p_t)|$ if $p_t$ possesses that confidence bound. As a consequence, we have similar upper bound on the performance loss comparing with Case (1), up to constant coefficients. \qed
\end{enumerate}
\end{proof}
The detailed proof of \Cref{lemma:c20_to_performance} is presented in \Cref{app:subsec:proof_lemma_c20_to_performance}. Finally, we have a lemma that upper bounds the regret of proposing $p_t=\frac{\hat a}{2\hat b}$ under special conditions.
\begin{lemma}[Proposing $\frac{\hat a}{2\hat b}$]
    \label{lemma:corner_case}
    When $\gamma_t>\frac{\hat a + C_a\cdot\frac1{\sqrt{\tau}}}2+c$ and when $\hat r_{k,t}-\Delta_k(t)>0, \forall k=-M, -M+1,\ldots, M$, we have $p_t^*=\frac{a}{2b}$ and there exists a constant $C_{non}$ such that
    \begin{equation}
        \label{eq:lemma_corner_case}
        \begin{aligned}
            r_t(\frac{a}{2b})-r_t(\frac{\hat a}{2\hat b}) \leq C_{non}\frac1{\tau}.
        \end{aligned}
    \end{equation}
\end{lemma}

The intuition of \Cref{lemma:corner_case} is that $\frac{a}{2b}$ is the optimal price without censoring, and we only need to show that either the optimal price or $\frac{a}{2b}$ is not censored (which are equivalent as the revenue function is unimodal). We defer its proof to \Cref{app:subsec:proof_lemma_corner_case}. This lemma serves as the last puzzle of the proof. With all lemmas above, we upper bound the overall regret as follows:

\begin{equation}
\label{eq:total_regret_sketch}
    \begin{aligned}
        Regret=\sum_{t=1}^T r_t(p_t^*) - r_t(p)
        &\leq \tau\cdot a_{\max} p_{\max} + (2M+1)(1+N_0+N_1)\cdot a_{\max} p_{\max}\\
        &\quad + \sum_{t=1}^T (\max\{C_{in}, C_{out}\})(\frac1{N_{k_t}(t)} + \frac1{\tau}) +\sum_{t=1}^T C_{non}\cdot\frac1{\tau}\\
        (\text{ let }\tau=\sqrt{T})\rightarrow\quad&=\tilde{O}(\sqrt{T} + T^{\frac14} + \sum_{t=1}^T\frac1{N_{k_t}(t)} + \frac{T}{\tau})\\
        &=\tilde{O}(\sqrt{T} + \sum_{k=1}^{2M+1}\sum_{i_k=1}^{N_{k_t}(T)}\frac1{i_k})\\
        &=\tilde{O}(\sqrt{T} + T^{\frac14}\log T) = \tilde{O}(\sqrt{T}).\\
    \end{aligned}
\end{equation}
Here the first two rows are a decomposition of STAGE 1 (for $\tau$ rounds), STAGE 2 Initialization (for $2M+1$ rounds), STAGE 2 Case (a) and (b) (proposing $p_t=p_{k_t}$, \Cref{lemma:c20_to_performance}) and STAGE 2 Case (c) (proposing $p_t = \frac{\hat a}{2\hat b}$, \Cref{lemma:corner_case}). The fourth row is by re-classification of $\frac1{N_{k_t}(t)}$ according to $k$, which leads to a summation over harmonic series (since each $N_{k_t}(t)$ increases by 1 for the same class $k_t=k$). By applying a union bound, \Cref{eq:total_regret_sketch} holds with probability $$\Pr\geq 1-2\eta\delta - 6\eta\delta\cdot T(2M+1) \geq 1- 20\frac{c}{2(C_a+C_b\cdot p_{\max})}T^{5/4}\cdot \eta\delta.$$
Here the first part comes from \Cref{lemma:a_b_estimation_error}, and the second part comes from \Cref{lemma:r_derivative_estimation_error} for any $t\in[T]$ and $k\in\{-M, \ldots, M\}$. Let $\eta:=\frac{C_a + C_b\cdot p_{\max}}{10c\cdot T^{5/4}}$, and we show that \Cref{thm:regret} holds. \qed
\end{proof}
\begin{remark}
    This $\tilde{O}(\sqrt{T})$ regret upper bound is near-optimal up to $\log{T}$ factors, as it matches the $\Omega(\sqrt{T})$ information-theoretic lower bound proposed by \citet{broder2012dynamic} for a \emph{no-censoring} problem setting with linear noisy demand.
\end{remark}


\section{Discussions}
\label{sec:discussion}
Here we provide some insights on the current limitations and potential extension of this work to a broader field of research.

\paragraph{Extensions to Non-linear Demand Curve.}
In this work, we adopt a linear-and-noisy model for the potential demands, which is standard in dynamic pricing literature. Also, we utilize the unimodal property brought by this linear demand model, even after the censoring effect is imposed. If we generalize our methodologies to nonlinear demand functions, we have to carefully distinguish between potential local optima and saddle points that may also cause $r'_t(p)=0$ for some sub-optimal $p$. We conjecture an $\Omega(T^{\frac{m+1}{2m+1}})$ lower bound in that case, where $m$ is the order of smoothness. It is worth investigating whether the censoring effect will introduce new local optimals or swipe off existing ones in multimodal settings.

\paragraph{Generalization to Unbounded Noises.}
We assume the noise is bounded in a constant-width range. From the analysis in \Cref{sec:regret_analysis}, we know that the threshold of learning the optimal price in our problem setting is still the estimation of parameters. Therefore, this boundedness assumption streamlines the pure-exploration phase, facilitates the estimation of the parameters $b$ and $a$, and scales down the cumulative regret. While our methods and results can be extended to unbounded $O(\frac1{\log T})$-subGaussian noises by simple truncation, challenges remain for handling generic unbounded noises. Moreover, the problem can be more sophisticated with \emph{dual-censoring}, both from above by inventory--as we have discussed-- and from below by $0$, especially when considering unbounded noises.

\paragraph{Extensions to Non-Lipschitz Noise CDF.} In this work, we assume the noise CDF as a Lipschitz function as many pricing-related works did \citep{fan2021policy, tullii2024improved}. This assumption enables the local smoothness at $p_t^*$ and leads to a quadratic loss. However, this prevents us from applying our algorithm to non-Lipschitz settings, which even includes the noise-free setting. In fact, although we believe that a better regret rate exists for the noise-free setting, we have to state that the hardness of the problem is completely different with Lipschitz noises versus without it. Although a Lipschitz noise makes the observation ``more blur'', it also makes the revenue curve ``more smooth''. We would like to present an analog example from the feature-based dynamic pricing problem: When the Gaussian noise $\cN(0, \sigma^2)$ is either negligible (with $\sigma<\frac1T$, see \citet{cohen2020feature_journal}) or super significant (with $\sigma>1$, see \citet{xu2021logarithmic}), the minimax regret is $O(\log T)$. However, existing works can only achieve $O(\sqrt{T})$ regret when $\sigma\in[\frac1T, 1]$. We look forward to future research on our problem setting once getting rid of the Lipschitzness assumption.

\paragraph{Extensions to Contextual Pricing.}
In this work, we assume $a$ and $b$ are static, which may not hold in many real scenarios. \Cref{example_performance} serves as a good instance, showcasing significant fluctuations in popularity across different performances. A reasonable extension of our work would be modeling $a$ and $b$ as \emph{contextual} parameters. Similar modelings have been adopt by \citet{wang2021dynamic} and \citet{ban2021personalized} in the realm of personalized pricing research.

\paragraph{Societal Impacts.}
Our research primarily addresses a non-contextual pricing model that does not incorporate personal or group-specific data, thereby adhering to conventional fairness standards relating to temporal, group, demand and utility discrepancies as outlined by \citet{cohen2022price} and \citet{chen2023utility}. However, the non-stationarity of inventory levels could result in varying \emph{fulfillment rate} over time, i.e. the proportions of satisfied demands at $\{p_t^*\}$'s might be different for $t=1,2,\ldots, T$. This raises concern regarding unfairness in fulfillment rate \citep{spiliotopoulou2022fairness}, particularly for products of significant social and individual importance.

\section{Conclusions}
\label{sec:conclusion}
In this paper, we studied the online pricing problem with adversarial inventory constraints imposed over time series. We introduced an optimistic strategy and a C20CB algorithm that is capable of approaching the optimal prices from inventory-censored demands. Our algorithm achieves a regret guarantee of $\tilde{O}(\sqrt{T})$ with high probability, which is information-theoretically optimal. To the best of our knowledge, we are the first to address this adversarial-inventory pricing problem, and our results indicate that the demand-censoring effect does not substantially increase the hardness of pricing in terms of minimax regret.

\section*{Acknowledgement}
Jianyu Xu started this work as a Ph.D. student at University of California, Santa Barbara. Xi Chen would like to thank the support from NSF via the Grant IIS-1845444.
%
%
%
\bibliographystyle{splncs04}
\bibliography{ref.bib}

\newpage

\appendix
{\Huge Appendix}
\section{More Related Works}
\label{app:subsec:related_works}
Here we discuss more related works as a complement to \Cref{sec:related_works}.

\paragraph{Contextual pricing: Linear valuation and binary-censored demand.}
A surge of research has focused on \emph{feature-based} (or \emph{contextual}) dynamic pricing \citep{cohen2020feature_journal, amin2014repeated, miao2019context, liu2021optimal}.  These works considered situations where each pricing period is preceded by a context, influencing both the demand curve and noise distribution. Specifically, \citet{cohen2020feature_journal, javanmard2019dynamic, xu2021logarithmic} explored a linear valuation framework with known distribution noise, leading to binary customer demand outcomes based on price comparisons to their valuations. Expanding on this, \citet{golrezaei2019incentive, fan2021policy, luo2021distribution, xu2022towards} examined similar models but with unknown noise distributions. In another vein, \citet{ban2021personalized, wang2021dynamic, xu2024pricing} investigated personalized pricing where demand is modeled as a generalized linear function sensitive to contextual price elasticity. Many of these works on valuation-based contextual pricing also assumed a censored demand: The seller only observes a binary feedback determined by a comparison of price with valuation, instead of observing the valuation directly. However, it is important to differentiate between the linear (potential) demand model we assumed and their linear valuation models, and there exists no inclusive relationship to each other. 

\paragraph{Dynamic pricing under constraints.} 
A variety of research works have been devoted to the field of dynamic pricing under specific concerns, which restricts the stages and outcomes of the market. This includes resource allocation \citep{chen2021joint, vera2021online, xu2025joint} and price discrimination \citep{chen2021fairness, cohen2021dynamic, cohen2022price, xu2023doubly, karan2024designing} as two typical instances. In this work, the constraints are twofold: First, it blocks our observations to the real potential demand, leading to biased estimates. Second, it restricts us from fulfilling the potential demand, leading to shifted targets (the optimal prices).

\paragraph{Multi-armed bandits.} Multi-armed bandits (MAB) formalize sequential decision-making under uncertainty via the exploration–exploitation tradeoff, originating in the stochastic setting studied by \citet{lai1985asymptotically}. Modern finite-time analyses led to widely used algorithms such as UCB \citep{auer2002finite}, Thompson Sampling \citep{thompson1933likelihood}, and the EXP-3/4 family \citep{auer2002nonstochastic}. Contextual bandits extend MAB by conditioning rewards on observed features, enabling personalization and decision-making with side information. Representative algorithmic and theoretical foundations include Epoch greedy \citep{langford2007epoch}, LinUCB \citep{chu2011contextual} and Taming-the-monster \citep{agarwal2014taming}. \citet{badanidiyuru2013bandits} introduces a ``Bandits with knapsacks'' model that incorporates resource and budget limitations. A separated stream of ``Continuum-armed bandits'', introduced by \citet{kleinberg2004nearly}, generalize discrete action sets to continuous domains and exploit smoothness/metric structure. Their method is applicable to our problem setting but will lead to a sub-optimal $O(T^{3/5})$ regret. A recent model of ``Generative bandits'' \citep{xu2025online} allows new arms to be generated on the fly (at a certain cost) for free future reuse. From a more general perspective, the stream of works in ``Zeroth-order optimization (ZOO)'' \citep{wang2018stochastic} also belongs to bandits optimization.

\section{Proof Details}
\label{app:sec:proof_detail}

\subsection{Proof of \Cref{lemma:r_t_p}}
\label{app:subsec:proof_lemma_r_t_p}
\begin{proof}
    We prove each property sequentially.
\begin{enumerate}
    \item For $r_t(p)$, we have:
    \begin{equation}
        \label{equ:r_p_calculate}
        \begin{aligned}
            r_t(p)&= \E[p_t\cdot D_t|p_t=p]\\
            &= p\cdot\E[\min\{\gamma_t, a-bp_t+N_t\}|p_t=p]\\
            &=p\cdot\E[\ind[a-bp+N_t\leq\gamma_t]\cdot(a-bp+N_t)+\ind[a-bp+N_t>\gamma_t]\cdot \gamma_t]\\
            &=p\cdot\E[\ind[N_t\leq \gamma_t-a+bp]\cdot(a-bp+N_t) + \ind[N_t >\gamma_t-a+bp]\cdot\gamma_t]\\
            &=p\left(\int_{-c}^{\gamma_t-a+bp}(a-bp+x)f(x)dx + \int_{\gamma_t-a+bp}^c \gamma_t f(x)dx\right)\\
            &=p\left(\int_{-c}^c(a-bp+x)f(x)dx + \int_{\gamma_t-a+bp}^c(\gamma_t-(a-bp+x))f(x)dx\right)\\
            &=p\left((a-bp)\cdot\int_{-c}^cf(x)dx + \int_{-c}^cxf(x)dx\right.\\
            &\quad\left.+(\gamma_t-a+bp)\int_{\gamma_t-a+bp}^cf(x)dx-\int_{\gamma_t-a+bp}^cxf(x)dx\right)\\
            &=p(a-bp)+0+p(\gamma_t-a+bp)(1-F(\gamma_t-a+bp))-p\cdot(xF(x)-G(x))|_{\gamma_t-a+bp}^c\\
            &=p\gamma_t-p(\gamma_t-a+bp)F(\gamma_t-a+bp)\\
            &\quad -p(c-G(c)-F(\gamma_t-a+bp)\cdot(\gamma_t-a+bp)+G(\gamma_t-a+bp))\\
            &=p(\gamma_t-c+G(c)-G(\gamma_t-a+bp)).
        \end{aligned}
    \end{equation}
    
    Here we adopt the notation $f(x)$ as \emph{proximal derivatives} of $F(x)$. According to Rademacher's Theorem (see Section 3.5 of \citet{folland1999real}), given that $F(x)$ is Lipschitz, the measure of $x$ such that $f(x)$ does not exist is zero, hence the integral holds. Here the eighth line comes from $\int_{-c}^c f(x)dx = F(c) - F(-c) = 1$ and $\int_{-c}^{c}xf(x)dx=\E[x]=0$. Given the close form of $r_t(p)$, we derive the form of $r'_t(p)$.
    \\
    \item As we have assumed, $F(x)$ is $L_F$-Lipschitz, and therefore $F(\gamma_t-a+bp)$ is $b_{\max}L_F$-Lipschitz, and $bpF(\gamma_t-a+bp)$ is $(b_{\max} + b_{\max}^2p_{\max}L_F)$-Lipschitz. Also, we have $\frac{d G(\gamma_t-a+bp)}{dp}=b\cdot F(\gamma_t-a+bp)\in[0, b_{\max}]$. Let $L_r:=(2b_{\max} + b_{\max}^2p_{\max}L_F)$, and we know that $r'_t(p)$ is $L_r$-Lipschitz.
    \\
    \item On the one hand, we have
    \begin{equation}
        \label{eq:r'_first_part_non_increasing}
        \begin{aligned}
            \frac{d(\gamma_t - c + G(c)-G(\gamma_t-a+bp))}{d p}=-bF(\gamma_t - a + bp)\leq 0.
        \end{aligned}
    \end{equation}
    On the other hand, for any $\Delta_p>0$, we have
    \begin{equation}
        \label{eq:r'_second_part_non_increasing}
        \begin{aligned}
            &-[b(p+\Delta_p)\cdot F(\gamma_t-a+b(p+\Delta_p))]- [-(bpF(\gamma_t-a+bp))]\\
            =&-b\Delta_pF(\gamma_t-a+b(p+\Delta_p)) + bp(F(\gamma_t-a+bp) - F(\gamma_t-a+b(p+\Delta_p)))\\
            \leq& 0 + 0 = 0.
        \end{aligned}
    \end{equation}
    Since $r'_t(p)=\gamma_t - c + G(c)-G(\gamma_t-a+bp) - bpF(\gamma_t-a+bp)$, we know that both components are monotonically non-increasing.
    \\
    \item We first show the \emph{existence} of $p_t^*\in[0,\frac{a}{b}]$ such that $r'_t(p_t^*)=0$. Recall that $G(c + x) = G(c) + x$ for $\forall x>0$, and $G(c) - G(-c) = \int_{-c}^{c}F(\omega)d\omega\geq0$, and that $\gamma_t>2c>c$ as we assumed. Given those, we have:
    \begin{equation}
        \label{eq:r'_t_0_and_r'_t_a/b}
        \begin{aligned}
            r'_t(0)&=\gamma_t-c+G(c)-G(\gamma_t-a)\\
            &>\gamma_t-c+G(c)-G(-c)\\
            &>0.\\
            r'_t(\frac{a}{b})&=\gamma_t - C + G(c) - G(\gamma_t)-b\cdot\frac{a}{b}\cdot F(\gamma_t)\\
            &=\gamma_t - c + G(c) - G(c+(\gamma_t-c)) - a\cdot 1\\
            &=\gamma_t - c + G(c) - G(c) - (G(c) + (\gamma_t - c)) - a\\
            &=\gamma_t - c + 0 - (\gamma_t - c) - a\\
            &=-a<0.
        \end{aligned}
    \end{equation}
    Also, $r'_t(p)$ is Lipschitz as we proved above. Therefore, $\exists p_t^*\in(0, \frac{a}{b})$ such that $r'_t(p_t^*)=0$.
    \\
    Now we show the uniqueness of $p_t^*$. If there exists $0<p_t^*<q_t^*<\frac{a}{b}$ such that $r'_t(p_t^*) = r'_t(q_t^*)=0$, then it leads to.
    \begin{equation}
        \label{eq:uniqueness_p_t^*}
        \begin{aligned}
            &r'_t(p)\equiv 0, \forall p\in[p_t^*, q_t^*] \text{ due to the monotonicity of }r'_t(p)\\
            \Rightarrow\quad&r'_t(p)\text{ is differentiable in} (p_t^*, q_t^*)\\
            \Rightarrow\quad&r''_t(p)\equiv 0 , \forall p\in (p_t^*, q_t^*)\\
            \Rightarrow\quad&F(\gamma_t-a+bp)\text{ is differentiable with respect to} p \text{ on }(p_t^*, q_t^*)\\
            \Rightarrow\quad&f(\gamma_t-a+bp)\text{ exists on }(p_t^*, q_t^*)\\
            \Rightarrow\quad&-2bF(\gamma_t-a+bp)-b^2pf(\gamma_t-a+bp)\equiv0, \forall p\in(p_t^*, q_t^*)\\
            \Rightarrow\quad&\left\{
            \begin{aligned}
            &F(\gamma_t-a+bp)=0,\text{ and}\\
            &f(\gamma_t-a+bp)=0, \forall p\in(p_t^*, q_t^*)\\
            \end{aligned}
            \right.\\
            \Rightarrow\quad&0\leq F(\gamma_t-a+bp_t^*)\leq \lim_{p\rightarrow q_t^*-}F(\gamma_t-a+bp)=0\\
            \Rightarrow\quad& F(\omega)\equiv 0, \forall\omega<\gamma_t-a+bq_t^*\\
            \Rightarrow\quad&r'_t(p_t^*)=\gamma_t-c+G(c)-G(\gamma_t-a+bp_t^*)-bp_t^*\cdot F(\gamma_t-a+bp_t^*)\\
            &\qquad\quad=\gamma_t-c+G(c)-\int_{-\infty}^{\gamma_t-a+bp_t^*}F(\omega)d\omega-0\\
            &\qquad\quad=\gamma_t-c+G(c)\\
            &\qquad\quad\geq\gamma_t-c\\
            &\qquad\quad>0.
        \end{aligned}
    \end{equation}
    This leads to contradictions that $r'_t(p_t^*)=0$. Therefore, $p_t^*$ is unique. Given this, we know that $r_t(p)$ is unimodal, which increases on $(0, p_t^*)$ and decreases on $(p_t^*, \frac{a}{b})$.
    \\
    \item Since $p_t^*$ is unique, and $r'_t(p)$ is $L_r$-Lipschitz, we have:
    \begin{equation}
        r_t(p_t^*)-r_t(p)\leq\frac{L_r}2\cdot(p_t^*-p)^2, \forall p\in[0, \frac{a}{b}].
    \end{equation}
    \item From the proof of part (4), we know that $F(\gamma_t-a+bp_t^*)>0$ (or otherwise $r'_t(p_t^*)>0$ leading to contradiction). Denote $\epsilon_t:=\frac{F(\gamma_t-a+bp_t^*)}{2L_Fb_{\max}}$, and we have:
    \begin{equation}
        \label{eq:epsilon_t}
        \begin{aligned}
            F(\gamma_t-a+b(p_t^*-\epsilon_t))\geq&F(\gamma_t-a+bp_t^*)-L_F\cdot b\cdot\epsilon_t\\
            \geq&F(\gamma_t-a+bp_t^*)-L_F\cdot b \cdot\frac{F(\gamma_t-a+bp_t^*)}{2L_F\cdot b_{\max}}\\
            =&\frac{F(\gamma_t-a+bp_t^*)}2.
        \end{aligned}
    \end{equation}
    Let $C_{\epsilon}:=\frac{b_{\min}}2\cdot\inf_{\gamma_t\in[\gamma_{\min}, \gamma_{\max}]} F(\gamma_t-a+bp_t^*)$. As $[\gamma_{\min}, \gamma_{\max}]$ is a close set and $F(\gamma_t-a+bp_t^*)$ holds for any $\gamma_t\in[\gamma_{\min}, \gamma_{\max}]$, we know that $C_{\epsilon}>0$ is a universal constant. Given this coefficient, for any $p_1, p_2\in[p_t^*-\epsilon_t, p_t^*+\epsilon_t], p_1<p_2$, we have
    \begin{equation}
        \label{eq:locally_strongly_concave}
        \begin{aligned}
            &r'_t(p_1) - r'_t(p_2)\\
            =&-(G(\gamma_t-a+bp_1)-G(\gamma_t-a+bp_2))-(bp_1F(\gamma_t-a+bp_1)-bp_2F(\gamma_t-a+bp_2))\\
            \geq&-(G(\gamma_t-a+bp_1)-G(\gamma_t-a+bp_2))\\
            =&\int_{\gamma_t-a+bp_1}^{\gamma_t-a+bp_2}F(\omega)d\omega\\
            \geq&\min_{\omega\in[p_t^*-\epsilon_t, p_t^*+\epsilon_t]}F(\omega)\cdot b(p_2-p_1)\\
            \geq&F(\gamma_t-a+b(p_t^*-\epsilon_t))\cdot b(p_2-p_1)\\
            >&\frac{F(\gamma_t-a+bp_t^*)}2\cdot b(p_2-p_1)\\
            \geq&C_{\epsilon}\cdot(p_2-p_1).
        \end{aligned}
    \end{equation}
    Here the third line is because $0<p_1<p_2$ and therefore $0< F(\gamma_t-a+bp_1)\leq F(\gamma_t-a+bp_2)$.
    \\
    \item According to part (6), for any $p\in(p_t^*-\epsilon_t, p_t^*+\epsilon_t)$, we have $|r'_t(p)|=|r'_t(p)-r'_t(p_t^*)|\geq C_{\epsilon}\cdot|p-p_t^*|$. Therefore, we have
    \begin{equation}
        \label{eq:loss_leq_derivative_square}
        \begin{aligned}
            (r'_t(p))^2\geq C_{\epsilon}^2\cdot(p-p_t^*)^2\geq\frac{C_{\epsilon}^2}{\frac{L_r}2}\cdot(r_t(p_t^*)-r_t(p))=\frac{2C_{\epsilon}^2}{L_r}\cdot(r_t(p_t^*)-r_t(p)).
        \end{aligned}
    \end{equation}
    Let $C_v:=\frac{L_r}{2C_{\epsilon}^2}$ and the property is proven.
\end{enumerate}
\end{proof}

\subsection{Proof of \Cref{lemma:a_b_estimation_error}}
\label{app:subsec:proof_lemma_a_b_estimation_error}
\begin{proof}
    Recall that $\gamma_0:=a_{\max} - b_{\min}p_{\max} + c$. Notice that
    \begin{equation}
        \label{eq:e_1_t-e_2_t_fraction_range}
        \begin{aligned}
            \frac{e_{1,t}-e_{2,t}}{\gamma_t - \gamma_0}\in[0, \frac1{\gamma_{\min}-\gamma_0}], e_{3,t}\in[0,1].
        \end{aligned}
    \end{equation}
    Also, for any $t=1,2,\ldots, \tau$ in STAGE 1, and any $\gamma$ such that $\gamma_{\min}\leq \gamma_t$, we have
    \begin{equation}
        \label{eq:expectation_indicator_primary}
        \begin{aligned}
            \E[\ind[D_t\geq \gamma]]=&\E_{N_t}[\E_{p_t\sim U[0, p_{\max}]}[a-bp_t+N_t\geq \gamma|N_t]]\\
            =&\E_{N_t}[\E_{p_t\sim U[0, p_{\max}]}[p_t\leq\frac{a-\gamma+N_t}{b}|N_t]]\\
            =&\E_{N_t}[\frac{a-\gamma+N_t}{b\cdot p_{\max}}]\\
            =&\frac{a-\gamma}{bp_{\max}}.
        \end{aligned}
    \end{equation}
    Here the first row is due to Law of Total Expectation, and the last row is due to the zero-mean assumption of $N_t$ (see \Cref{assumption:noise}). Given this equation, we have:
    \begin{equation}
        \label{eq:expectation_e_i_t_results}
        \begin{aligned}
            \E[\frac{e_{1,t}-e_{2,t}}{\gamma_t-\gamma_0}]&=\frac{\frac{\frac{2\gamma_t + 2\gamma_0}4 - \frac{1\gamma_t + 3\gamma_0}4}{bp_{\max}}}{\gamma_t - \gamma_0}=\frac{\frac14\cdot(\gamma_t-\gamma_0\cdot\frac1{bp_{\max}})}{\gamma_t-\gamma_0}=\frac1{4bp_{\max}}.\\
            \E[e_{3,t}]&=\frac{a-\frac{3\gamma_t+\gamma_0}4}{bp_{\max}}.
        \end{aligned}
    \end{equation}
    According to Hoeffding's Inequality, we have with $\Pr\geq 1-\eta\delta$:
    \begin{equation}
        \label{eq:hoeffding_b}
        \begin{aligned}
            |\frac1{\tau}\sum_{t=1}^{\tau}\frac{e_{1,t}-e_{2,t}}{\gamma_t - \gamma_0} - \frac1{4bp_{\max}}|\leq\frac1{\gamma_{\min}-\gamma_0}\sqrt{\frac12\log\frac{2}{\eta\delta}}\cdot\frac1{\sqrt{\tau}}.
        \end{aligned}
    \end{equation}
    Based on this concentration, we upper bound the estimation error between $b$ and $\hat b$ by the end of STAGE 1:
    \begin{equation}
        \label{eq:error_b}
        \begin{aligned}
            |\hat b - b|&=|\frac1{4p_{\max}\cdot\frac1{\tau}\sum_{t=1}^{\tau}\frac{e_{1,t}-e_{2,t}}{\gamma_t - \gamma_0}} - b|\\
            &=|\frac1{4p_{\max}(\frac1{\tau}\sum_{t=1}^{\tau}\frac{e_{1,t}-e_{2,t}}{\gamma_t - \gamma_0} - \frac1{4bp_{\max}}) + \frac1b}-\frac1{\frac1b}|\\
            &=|\frac{4p_{\max}\cdot\frac1{\tau}\sum_{t=1}^{\tau}\frac{e_{1,t}-e_{2,t}}{\gamma_t - \gamma_0}\cdot\frac1{4bp_{\max}}}{4p_{\max}\cdot\frac1{\tau}\sum_{t=1}^{\tau}\frac{e_{1,t}-e_{2,t}}{\gamma_t - \gamma_0}}|\\
            &\leq\frac{4p_{\max}\frac1{\gamma_{\min}-\gamma_0}\sqrt{\frac12\log\frac2{\eta\delta}}\cdot\frac1{\sqrt{\tau}}}{4p_{\max}(\frac1{4bp_{\max}}-\frac1{\gamma_{\min}-\gamma_0}\sqrt{\frac12\log\frac2{\eta\delta}}\cdot\frac1{\sqrt{\tau}})\cdot\frac1b}\\
            &\leq\frac{\frac1{\gamma_{\min}-\gamma_0}\sqrt{\frac12\log\frac2{\eta\delta}}\cdot\frac1{\sqrt{\tau}}}{\frac1{8bp_{\max}}\cdot\frac1b}\\
            &\leq\frac{4b_{\max}^2p_{\max}}{\gamma_{\min} - \gamma_0}\cdot\sqrt{\frac12\log\frac2{\eta\delta}}\cdot\frac1{\sqrt{\tau}}.
        \end{aligned}
    \end{equation}
    Here the fifth row requires $\frac1{8bp_{\max}}\geq \frac1{\gamma_{\min}-\gamma_0}\sqrt{\frac12\log\frac2{\eta\delta}}\cdot\frac1{\sqrt{\tau}}$, which further requires $T\geq (\frac{8bp_{\max}}{\gamma_{\min}-\gamma_0})^4\cdot\frac14\log^2\frac2{\eta\theta}$. According to \Cref{assumption:large_T}, we know that this inequality holds. Denote $C_b:=\frac{4b_{\max}^2p_{\max}}{\gamma_{\min} - \gamma_0}\cdot\sqrt{\frac12\log\frac2{\eta\delta}}$ and we have $|\hat b - b|\leq C_b\cdot\frac1{\sqrt{\tau}}$ with high probability.

    Again, according to Hoeffding's Inequality, we have with $\Pr\geq 1-\eta\delta$:
    \begin{equation}
        \label{eq:hoeffding_a}
        \begin{aligned}
            |\frac1\tau\sum_{t=1}^\tau e_{3,t} - \frac{a-\frac{3\gamma_t + \gamma_0}4}{bp_{\max}}|\leq\sqrt{\frac12\log\frac2{\eta\delta}}\cdot\frac1{\sqrt{\tau}}.
        \end{aligned}
    \end{equation}
    Hence we have
    \begin{equation}
        \label{eq:error_a}
        \begin{aligned}
            |\hat a - a|=&|\frac1\tau\sum_{t=1}^\tau(\hat b p_{\max}e_{3,t} + \frac{3\gamma_t + \gamma_0}4) - a|\\
            =&|(\hat b - b)p_{\max}\cdot\frac1{\tau}\sum_{t=1}^\tau e_{3,t} + \frac1{\tau}\sum_{t=1}^\tau(b p_{\max}e_{3,t} - (a-\frac{3\gamma_t + \gamma_0}4))|\\
            \leq & |\hat b - b|p_{\max}\cdot 1 + |\frac1{\tau}\sum_{t=1}^\tau bp_{\max}e_{3,t} - (a-\frac{3\gamma_t + \gamma_0}4)|\\
            \leq & p_{\max}|\hat b - b| + bp_{\max}\cdot\sqrt{\frac12\log\frac2{\eta\delta}}\cdot\frac1{\sqrt{\tau}}\\
            =&p_{\max}\cdot C_b\cdot p_{\max}\frac1{\sqrt{\tau}} + bp_{\max}\sqrt{\frac12\log\frac2{\eta\delta}}\cdot\frac1{\sqrt{\tau}}\\
            \leq&p_{\max}(C_b + b_{\max}\sqrt{\frac12\log\frac2{\eta\delta}})\cdot\frac1{\sqrt{\tau}}.
        \end{aligned}
    \end{equation}
    Denote $C_a:=p_{\max}(C_b + b_{\max}\sqrt{\frac12\log\frac2{\eta\delta}})$ and the lemma is proven.
\end{proof}

\subsection{Proof of \Cref{lemma:r_derivative_estimation_error}}
\label{app:subsec:proof_lemma_r_derivative_estimation_error}
\begin{proof}
    Here we consider the time periods before time $t$, and we use an index $s$ to denote each time period $s=1,2,\ldots, t-1, t$. As a consequence, we have the notations $D_s, \gamma_s, k_s, \ind_s$ corresponding to $D_t, \gamma_t, k_t, \ind_t$ as we defined in \Cref{sec:preliminaries} and \Cref{sec:algorithm}. Also, we denote $N_k(t), F_k(t)$ and $G_k(t)$ as the value of $N_k, F_k$ and $G_k$ at the beginning of time period $t$.

    From \Cref{algo:C20CB}, we have
    \begin{equation}
        \label{eq:r'_estimate_decomposition_1}
        \begin{aligned}
            |\hat r_{k,t} - r'_t(p_{k,t})|\leq&|G_k(t)-(G(c)-G(\gamma_t-a+bp_{k,t}))|+p_{k,t}|\hat bF_k-bF(\gamma_t-a+bp_{k,t})|.
        \end{aligned}
    \end{equation}
    Notice that $G_k(t)=\frac1{N_k(t)}\cdot\sum_{s=1}^{t-1}\ind[k_s==k]\cdot(D_s-\gamma_s+c)$. Also, for each $D_s$ on the price $p_{k_s, s}$, we have
    \begin{equation}
        \label{eq:expectation_d_s}
        \begin{aligned}
            \E[D_s|p_{k_s, s}=\frac{W_k-(\gamma_s-\hat a)}{\hat b}]=(\gamma_t-c+G(c)-G(\gamma_s-a+bp_{k_s,s})).
        \end{aligned}
    \end{equation}
    Recall that $W_k=2k\Delta$. Hence we have
    \begin{equation}
        \label{eq:expectation_g_k_t}
        \begin{aligned}
            \E[G_k(t)]=&\frac1{N_k(t)}\sum_{s=1}^{t-1}\ind[k_s==k]\cdot\E[D_s-\gamma_s+c]\\
            =&\frac1{N_k(t)}\sum_{s=1}^{t-1}\ind[k_s==k](G(c)-G(\gamma_s-a+b\cdot\frac{2k\Delta-\gamma_s+\hat a}{\hat b})).
        \end{aligned}
    \end{equation}
    Also, due to the fact that $G(x)$ is $1$-Lipschitz, we have
    \begin{equation}
        \label{eq:g_k_estimation_error}
        \begin{aligned}
            &|G(\gamma_s-a+b\cdot\frac{2k\Delta-\gamma_s+\hat a}{\hat b}) - G(2k\Delta)|\\
            =&|G(\gamma_s-a+b\cdot\frac{2k\Delta-\gamma_s+\hat a}{b} +b(\frac1{\hat b}-\frac1b)(2k\delta-\gamma_s+\hat a)) - G(2k\Delta)|\\
            =&|G(\gamma_s-a+2k\Delta-\gamma_s+\hat a +\frac{b-\hat b}{\hat b}\cdot(2k\Delta-\gamma_s+\hat a)) - G(2k\Delta)|\\
            =&|G(2k\Delta + (\hat a - a) + (b-\hat b)\cdot\frac{2k\Delta-\gamma_s+\hat a}{\hat b}) - G(2k\Delta)|\\
            \leq&|2k\Delta + (\hat a - a) + (b-\hat b)\cdot\frac{2k\Delta-\gamma_s+\hat a}{\hat b}-2k\Delta|\\
            \leq&|\hat a-a| + |b-\hat b|\cdot\frac{|2k\Delta-\gamma_s+\hat a|}{\hat b}\\
            <&(C_a + C_b\cdot\frac{c+a_{\max}}{b_{\min}})\cdot\frac1{\sqrt{\tau}}.
        \end{aligned}
    \end{equation}
    Therefore, we know that
    \begin{equation}
        \label{eq:g_k_bias}
        \begin{aligned}
            |\E[G_k(t)]-(G(c)-G(2k\Delta))|\leq&\frac1{N_k(t)}\sum_{s=1}^{t-1}\ind[k_s==k]|G(\gamma_s - a + b\cdot\frac{2k\Delta - \gamma_s + \hat a}{\hat b}) - G(2k\Delta)|\\
            \leq&\frac1{N_k(t)}\sum_{s=1}^{t-1}\ind[k_s==k](C_a+C_b\cdot\frac{c+a_{\max}}{b_{\min}})\cdot\frac1{\sqrt{\tau}}\\
            \leq&(C_a+C_b\cdot\frac{c+a_{\max}}{b_{\min}})\cdot\frac1{\sqrt{\tau}}.
        \end{aligned}
    \end{equation}
    Also, since each $D_s-\gamma_s + c < a_{max} + c$, according to Hoeffding's Inequality, with $\Pr\geq 1- \eta\delta$ we have
    
    \begin{equation}
        \label{eq:hoeffding_g_k}
        \begin{aligned}
            &|G_k(t)-(G(c)-G(2k\Delta))|\\
            \leq&|G_k(t)-\E[G_k(t)]| + |\E[G_k(t)]-(G(c)-G(2k\Delta))|\\
            \leq&(a_{\max}+c)\cdot\sqrt{\frac12\log\frac2{\eta\delta}}\frac1{\sqrt{N_k{t}}} + (C_a+C_b\cdot\frac{c+a_{\max}}{b_{\min}})\cdot\frac1{\sqrt{\tau}}.
        \end{aligned}
    \end{equation}
    On the other hand, since
    \begin{equation}
        \label{eq:expectation_f_k}
        \begin{aligned}
            \E[F_k(t)]=&\frac1{N_k(t)}\sum_{s=1}^{t-1}\ind[k_s==k]\E[\ind[D_s<\gamma_s]]\\
            =&\frac1{N_k(t)}\sum_{s=1}^{t-1}\ind[k_s==k]\cdot F(\gamma_s-a+b\cdot\frac{2k\Delta-\gamma_s+\hat a}{\hat b}).
        \end{aligned}
    \end{equation}
    Similar to \Cref{eq:g_k_estimation_error}, since $F(x)$ is $L_F$-Lipschitz, we have
    \begin{equation}
        \label{eq:f_k_estimation_error}
        \begin{aligned}
            &|F(\gamma_s-a+b\cdot\frac{2k\Delta-\gamma_s+\hat a}{\hat b}) - F(2k\Delta)|\\
            =&|F(\gamma_s-a+b\cdot\frac{2k\Delta-\gamma_s+\hat a}b + b(\frac1{\hat b}-\frac1b)(2k\Delta-\gamma_s+\hat a))-F(2k\Delta)|\\
            =&|F(\gamma_s-a+2k\Delta-\gamma_s+\hat a+\frac{b-\hat b}{\hat b}(2k\Delta-\gamma_s+\hat a))-F(2k\Delta)|\\
            =&|F(2k\Delta+(\hat a-a) + (b-\hat b)\cdot\frac{2k\Delta-\gamma_s+\hat a}{\hat b})-F(2k\Delta)|\\
            \leq&L_F\cdot(|\hat a - a| + |b - \hat b|\cdot\frac{2k\Delta-\gamma_s+\hat a}{\hat b})\\
            <&L_F\cdot(C_a + C_b\cdot\frac{c+a_{\max}}{b_{\min}})\cdot\frac1{\sqrt{\tau}}.
        \end{aligned}
    \end{equation}
    Therefore, we have
    \begin{equation}
        \label{eq:f_k_bias}
        \begin{aligned}
            |\E[F_k(t)-F(2k\Delta)]|\leq L_F\cdot(C_a + C_b\cdot\frac{c+a_{\max}}{b_{\min}})\cdot\frac1{\sqrt{\tau}}.
        \end{aligned}
    \end{equation}

    Also, since $\ind_s=\ind[D_s<\gamma_s]$, according to Hoeffding's inequality, with $\Pr\geq1-\eta\delta$ we have
    \begin{equation}
        \label{eq:hoeffding_f_k}
        \begin{aligned}
            &|F_k(t)-F(2k\Delta)|\\
            \leq&|F_k(t)-\E[F_k(t)]| + |\E[F_k(t)] - F(2k\Delta)|\\
            \leq&1\cdot\sqrt{\frac12\log\frac2{\eta\delta}}\cdot\frac1{N_k(t)} + L_F(C_a + C_b\cdot\frac{c+a_{\max}}{b_{\min}})\cdot\frac1{\sqrt{\tau}}.
        \end{aligned}
    \end{equation}
    As a consequence, we have bounded the estimation error of $\hat r_{k,t}$ from $r'_t(p_{k,t})$:
    \begin{equation}
        \label{eq:r_derivatives_estimation_error_final}
        \begin{aligned}
            &|r'_t(p_{k,t})-\hat r_{k,t}|\\
            \leq&|\gamma_t - c + G_k(t)-\hat b p_{k,t} F_k(t) - (\gamma_t-c+G(c)-G(2k\Delta)-b p_{k,t} F(2k\Delta))|\\
            &+ |G(2k\Delta) - G(\gamma_t-a+b\cdot\frac{2k\Delta-(\gamma_t-\hat a)}{\hat b})| \\
            &+ bp_{k,t}|F(2k\Delta) - F(\gamma_t-a+b\cdot\frac{2k\Delta-(\gamma_t-\hat a)}{\hat b})|\\
            \leq&|G_k(t)-(G(c)-G(2k\Delta))| + |\hat b - b|\cdot p_{k,t}F_k + bp_{k,t}|F_k(t)-F(2k\Delta)|\\
            &+ |G(2k\Delta)-G(\gamma_t-a+b\cdot\frac{2k\Delta-(\gamma_t-\hat a)}{\hat b})| \\
            &+ b_{\max}p_{\max}\cdot|F(2k\Delta)-F(\gamma_t-a+b\cdot\frac{2k\Delta - (\gamma_t-\hat a)}{\hat b})|\\
            \leq&(c+a_{\max})\sqrt{\frac12\log\frac2{\eta\delta}}\cdot\frac1{\sqrt{N_k(t)}} + (C_a + C_b\cdot\frac{c+a_{\max}}{b_{\min}})\cdot\frac1{\sqrt{\tau}} + p_{\max}C_b\cdot\frac1{\sqrt{\tau}}\\
            &+ b_{\max}p_{\max}(\sqrt{\frac12\log\frac2{\eta\delta}}\cdot\frac1{\sqrt{N_k(t)}} + L_F(C_a + C_b\cdot\frac{c+a_{\max}}{b_{\min}})\cdot\frac1{\sqrt{\tau}})\\
            &+ (C_a + C_b\cdot\frac{c+a_{\max}}{b_{\min}})\cdot\frac1{\sqrt{\tau}} + b_{\max}p_{\max}L_F(C_a + C_b\cdot\frac{c+a_{\max}}{b_{\min}})\cdot\frac1{\tau}\\
            =:&C_N\cdot\frac1{\sqrt{N_k(t)}} + C_{\tau}\cdot\frac1{\sqrt{\tau}}
            =\Delta_k.
        \end{aligned}
    \end{equation}
    Here 
    \begin{equation}
        \label{eq:def_c_n_c_tau}
        \begin{aligned}
            C_N&:=(c+a_{\max} + b_{\max}p_{\max})\cdot\sqrt{\frac12\log\frac2{\eta\delta}},\\
            C_{\tau}&:=2(b_{\max}p_{\max}L_F + 1)(C_a + C_b\cdot\frac{c+a_{\max}}{b_{\min}}) + p_{\max}C_b.
        \end{aligned}
    \end{equation}
    Finally, we apply the union bound on the probability, and know that \Cref{eq:r_derivatives_estimation_error_final} holds with probability $\Pr\geq 1- 6\eta\delta$.
\end{proof}

\subsection{Proof of \Cref{lemma:c20_to_performance}}
\label{app:subsec:proof_lemma_c20_to_performance}
\begin{proof}
    We first prove the lemma under Case 1 when some confidence bound contains $0$. Denote $\rho_t:=\min\{|r'_t(p_t^*-\epsilon_t)|, |r'_t(p_t^*+\epsilon_t)|\}$, and we know that $\rho_t>0$ due to the uniqueness of $p_t^*$.

    Now, let $N_0:=\frac{36C_N^2}{\rho_t^2}$, where $C_N$ is the constant coefficient define in \Cref{lemma:r_derivative_estimation_error}. Given this, when $N_k(t)\geq N_0$, we have
    \begin{equation}
        \label{eq:delta_k_t_bounded_by_rho_part_1}
        \begin{aligned}
            C_N\cdot\frac1{\sqrt{N_k(t)}}&\leq C_N\frac1{\sqrt{N_0}}\\
            &=C_N\cdot\frac1{\frac{6C_N}{\rho_t}}\\
            &=C_N\cdot\frac{\rho_t}{6C_N}=\frac{\rho_t}6.
        \end{aligned}
    \end{equation}
    Also, since $T$ is assumed as larger than any constant (see \Cref{assumption:large_T}), we have $T\geq\frac{1296C_{\tau}^4}{\rho_t^4}$, and therefore
    \begin{equation}
        \label{eq:delta_k_t_bounded_by_rho_part_2}
        \begin{aligned}
            C_{\tau}\cdot\frac1{\sqrt{\tau}}&= C_{\tau}\frac1{T^{1/4}}\\
            &\leq C_{\tau}\cdot\frac1{\frac{6C_{\tau}}{\rho_t}}\\
            &=C_N\cdot\frac{\rho_t}{6C_{\tau}}=\frac{\rho_t}6.
        \end{aligned}
    \end{equation}
    Given \cref{eq:delta_k_t_bounded_by_rho_part_1} and \cref{eq:delta_k_t_bounded_by_rho_part_2}, we know that $\Delta_k(t) = C_N\cdot\frac1{\sqrt{N_k(t)}} + C_{\tau}\cdot\frac1{\sqrt{\tau}}\leq\frac{\rho_t}6 + \frac{\rho_t}6 = \frac{\rho_t}3$. Now, if $0\in[\hat r_{k,t}-\Delta_k(t), \hat r_{k,t}+\Delta_k(t)]$, we have $|\hat r_{k,t}|\leq\Delta_k(t)$ and therefore
    \begin{equation}
        \label{eq:case_1_derivative_upper_bound_by_delta_k_t}
        \begin{aligned}
            |r'_t(p_{k,t})|\leq|r'_t(p_{k,t}) - \hat r_{k,t}| + |\hat r_{k,t} - 0| \leq \Delta_k(t) + \Delta_k(t) \leq \frac{\rho_t}3 + \frac{\rho_t}3 = \frac{2\rho_t}3.
        \end{aligned}
    \end{equation}
    Since $r'_t(p)$ is monotonically non-increasing, any $p\in[0, p_{\max}]$ satisfying $r'_t(p)<\rho_t$ should satisfy $p\in(p_t^*-\epsilon_t, p_t^*+\epsilon_t)$. Therefore, we have $p_{k,t}\in(p_t^*-\epsilon_t, p_t^*+\epsilon_t)$. According to \Cref{lemma:r_t_p} Property (7), we have:
    \begin{equation}
        \label{eq:per_round_regret_bound_case_1}
        \begin{aligned}
            |r_t(p_t^*) - r_t(p_{k,t})|\leq&C_v\cdot(r'_t(p_{k,t}))^2\\
            \leq&C_v\cdot(2\Delta_k(t))^2\\
            \leq&4C_v(C_N\cdot\frac1{\sqrt{N_k(t)}} + C_{\tau}\cdot\frac1{\sqrt{\tau}})^2\\
            \leq&8C_v(C_N^2\cdot\frac1{N_k(t)} + C_{\tau}^2\cdot\frac1{\tau}).
        \end{aligned}
    \end{equation}
    Let $C_{in}:=8C_v\max\{C_N^2, C_{\tau}^2\}$ and the first part of \Cref{lemma:c20_to_performance} holds.
    \\
    Now let us prove the lemma under Case 2 when \emph{no} confidence bound contains 0. Formally stated, we have
    \begin{equation}
        \label{eq:c20_to_performance_case_2_formally}
        \begin{aligned}
            (\hat r_{k,t} + \Delta_k(t))(\hat r_{k,t} - \Delta_k(t))>0, \forall k=-M, -M+1, \ldots, M-1, M.
        \end{aligned}
    \end{equation}
    
    Denote $\theta_t:=\inf_k\min\{|\hat r_{k,t} + \Delta_k(t)|, |\hat r_{k,t} - \Delta_k(t)|\}$, and we know that $\min\{|\hat r_{k_t,t} + \Delta_k(t)|, |\hat r_{k_t,t} - \Delta_k(t)|\}\theta_t>0$, where $k_t$ is the $k$ such that $p_{k, t}$ is proposed at time $t$. Therefore, we have $|\hat r_{k,t} + \Delta_k(t)|\geq\theta_t$ and $|\hat r_{k,t} - \Delta_k(t)|\geq\theta_t, \forall k$. According to the prerequisite of \Cref{lemma:c20_to_performance} Part (2), there exists $k_0$ such that
    \begin{equation}
        \label{eq:k_0_+_-}
        \left\{
        \begin{aligned}
           &\hat r_{k_0,t} - \Delta_{k_0}(t) \geq \theta_t\\
           &\hat r_{k_0+1, t} + \Delta_{k_0+1}(t) \leq \theta_t.
        \end{aligned}
        \right.
    \end{equation}
    Also, since $r'_t(p)$ is $L_r$-Lipschitz, we have
    \begin{equation}
        \label{eq:theta_t_upper_bound}
        \begin{aligned}
            |r'_t(p_{k_0, t}) - r'_t(p_{k_0+1, t})|&\leq L_r(p_{k_0+1, t} - p_{k_0, t}) = L_r\frac{2\Delta}{\hat b}\\
            \Rightarrow\quad 2L_r\frac{\Delta}{\hat b}&\geq|r'_t(p_{k_0, t}) - r'_t(p_{k_0+1, t})|\geq 2\theta_t\\
            \Rightarrow\quad \theta_t&\leq\frac{L_r}{\hat b}\cdot\Delta\leq\frac{L_r}{b_{\min}}(C_a + C_bp_{\max})\cdot\frac1{T^{1/4}}.
        \end{aligned}
    \end{equation}
    As $T$ is sufficiently large, we have $\frac{L_r}{b_{\min}}(C_a + C_bp_{\max})\cdot\frac1{T^{1/4}}\leq\frac16\cdot\rho_t$ where $\rho_t:=\min\{|r'_t(p_t^*-\epsilon_t)|, |r'_t(p_t^*+\epsilon_t)|\}$. Let $N_1:=\frac{144C_N^2}{\rho_t^2}$. Similar to \Cref{eq:delta_k_t_bounded_by_rho_part_1} and \Cref{eq:delta_k_t_bounded_by_rho_part_2}, we have
    \begin{equation}
        \label{eq:delta_k_t_bounded_by_rho_similar}
        \begin{aligned}
            C_N\cdot\frac1{\sqrt{N_{k_t}(t)}}\leq&\frac1{12}\rho_t\\
            C_{\tau}\cdot\frac1{\sqrt{\tau}}\leq&\frac1{12}\rho_t.
        \end{aligned}
    \end{equation}
    Hence, we have
    \begin{equation}
        \label{eq:derivative_bound_by_rho_case_2}
        \begin{aligned}
            |r'_t(p_{k_t, t})|\leq2\Delta_{k_t}(t) + \theta_t\leq 2\cdot(\frac1{12}\rho_t + \frac1{12}\rho_t) + \frac{\rho_t}6=\frac{\rho_t}2.
        \end{aligned}
    \end{equation}
    Since $r'_t(p)$ is monotonically non-increasing, we know that $p_{k_t, t}\in(p_t^*-\epsilon_t, p_t^*+\epsilon_t)$ similar to the analysis in Case (1), and again we have $|r_t(p_t^*)-r_t(p_{k_t, t})|\leq 8C_v(C_N^2\cdot\frac1{N_{k_t}(t)}+C_{\tau}^2\cdot\frac1{\tau})$. This completes the proof of \Cref{lemma:c20_to_performance} on both circumstances.
\end{proof}

\subsection{Proof of \Cref{lemma:corner_case}}
\label{app:subsec:proof_lemma_corner_case}
\begin{proof}
    When $\gamma_t>\frac{\hat a + C_a\cdot\frac1{\sqrt{\tau}}}2+c>\frac{a}2+c$, we know that the demand at $p=\frac{a}{2b}$ and its neighborhood is not censored. As a result, $\frac{a}{2b}$ is still a local optimal (and therefore global optimal due to the unimodality) of $r_t(p)$. According to \Cref{lemma:a_b_estimation_error}, we have:
    \begin{equation}
        \label{eq:regret_corner_case_1}
        \begin{aligned}
            r_t(\frac{a}{2b})-r_t(\frac{\hat a}{2\hat b})&\leq C_s(\frac{a}{2b}-\frac{\hat a}{2\hat b})^2\\
            &= C_s((\frac{a}{2b}-\frac{a}{2\hat b}) + (\frac{a}{2\hat b}-\frac{\hat a}{2\hat b}))^2\\
            &\leq C_s(\frac{a|\hat b - b|}{2b\hat b} + \frac{a-\hat a}{2\hat b})^2\\
            &\leq 2C_s(\frac{a_{\max}^2}{4b_{\min}^4} (\hat b - b)^2 + \frac1{4b_{\min}^2}(a-\hat a)^2)\\
            &\leq \frac{2C_s(a_{\max}^2 C_b^2 + b_{\min}^2C_a^2)}{4b_{\min}^4}\cdot\frac1{\tau}.
        \end{aligned}
    \end{equation}
    When $\hat r_{k,t}-\Delta_k(t)>0, \forall k=-M, -M+1,\ldots, M-1, M$, we know that $r'_t(\frac{a+c-\gamma_t}b)>0$ according to \Cref{lemma:r_derivative_estimation_error}. Since $d_t(\frac{a+c-\gamma_t}b)=a-b\cdot\frac{a+c-\gamma_t}b+N_t = \gamma_t-c+N_t\leq\gamma_t$ is not censored, we know that the optimal price $p_t^*$ satisfies $p_t^*>\frac{a+c-\gamma_t}b$ (since $r'_t(p_t^*)=0<r'_t(\frac{a+c-\gamma_t}b)$) and therefore its demand is not censored. Therefore, the optimal price $p_t^*=\frac{a}{2b}$ and we have its regret bounded by \Cref{eq:regret_corner_case_1} identically. Let $C_{non}:=\frac{2C_s(a_{\max}^2 C_b^2 + b_{\min}^2C_a^2)}{4b_{\min}^4}$ and we have proven both cases.
\end{proof}

\label{general:appendix}

%




\end{document}